\documentclass[11pt,a4paper]{article}

\usepackage{amsmath,amsthm,amssymb,amscd,mathtools}
\DeclarePairedDelimiter{\ceil}{\lceil}{\rceil}
\usepackage{hyperref}
\usepackage{floatrow}
\usepackage[affil-it]{authblk}
\usepackage{xcolor}
\usepackage[font=small]{caption}
\usepackage{comment}
\usepackage{graphicx}
\usepackage{subcaption}
\usepackage{graphics}
\usepackage{verbatim}
\usepackage{hyperref}
\usepackage{placeins}
\usepackage{longtable}
\usepackage{parskip}
\usepackage{caption}
\usepackage{mathrsfs}
\usepackage{color}
\usepackage{textgreek}
\usepackage[top=1in, bottom=1in, left=1.25in, right=1.25in]{geometry}
\usepackage{enumerate}
\usepackage{enumitem}
\pdfoutput=1

\usepackage[square,sort,comma,numbers]{natbib}
\setcitestyle{authoryear,open={(},close={)}}
\usepackage{dirtytalk}

\newcommand{\q}[1]{`#1'}
\DeclareTextFontCommand{\textmyfont}{\pcr}

\newtheorem{theorem}{Theorem}
\newtheorem{definition}{Definition}

\providecommand{\keywords}[1]{\textbf{\textit{Index terms---}} #1}

\DeclareCaptionFont{Small}{\fontsize{7}{11}\selectfont}
\usepackage{array}

\usepackage[ruled, vlined]{algorithm2e}
\newcommand{\var}{\texttt}
\makeatletter
\renewcommand{\SetKwInOut}[2]{%
  \sbox\algocf@inoutbox{\KwSty{#2\algocf@typo:}}%
  \expandafter\ifx\csname InOutSizeDefined\endcsname\relax
    \newcommand\InOutSizeDefined{}\setlength{\inoutsize}{\wd\algocf@inoutbox}%
    \sbox\algocf@inoutbox{\parbox[t]{\inoutsize}{\KwSty{#2\algocf@typo\hfill:}}~}\setlength{\inoutindent}{3em}
  \else
    \ifdim\wd\algocf@inoutbox>\inoutsize%
    \setlength{\inoutsize}{\wd\algocf@inoutbox}%
    \sbox\algocf@inoutbox{\parbox[t]{\inoutsize}{\KwSty{#2\algocf@typo\hfill:}}~}\setlength{\inoutindent}{3em}
    \fi%
  \fi
  \algocf@newcommand{#1}[1]{%
    \ifthenelse{\boolean{algocf@hanginginout}}{\relax}{\algocf@seteveryparhanging{\relax}}%
    \ifthenelse{\boolean{algocf@inoutnumbered}}{\relax}{\algocf@seteveryparnl{\relax}}%
    {\let\\\algocf@newinout\hangindent=\inoutindent\hangafter=1\parbox[t]{\inoutsize}{\KwSty{#2\algocf@typo\hfill:}}~##1\par}
    \algocf@linesnumbered
    \ifthenelse{\boolean{algocf@hanginginout}}{\relax}{\algocf@reseteveryparhanging}%
  }}%
\makeatother
\SetAlFnt{\footnotesize}
\SetAlCapFnt{\large}
\SetAlCapNameFnt{\large}

\begin{document}

\title{LoRAS: An oversampling approach for imbalanced datasets}

\renewcommand\Authfont{\fontsize{9}{14.4}\selectfont}
\renewcommand\Affilfont{\fontsize{7}{10.8}\itshape}

\author[1]{Saptarshi Bej}
\author[1,2,3]{Narek Davtyan}
\author[1]{Markus Wolfien}
\author[1]{Mariam Nassar}
\author[1,4]{Olaf Wolkenhauer \thanks{Corresponding author: Prof.\! Olaf Wolkenhauer, Department of Systems Biology \& Bioinformatics, University of Rostock, Universit\"atsplatz 1, 18051 Rostock, Germany. E-mail: \href{mailto:olaf.wolkenhauer@uni-rostock.de}{olaf.wolkenhauer@uni-rostock.de}. Website: \href{https://www.sbi.uni-rostock.de/}{sbi.uni-rostock.de}. }}
\affil[1]{Department of Systems Biology and Bioinformatics,
		University of Rostock, Germany}
\affil[2]{Department of Computer Science and Applied Mathematics, Grenoble Alps University, France}
\affil[3]{School of Computer Science and Applied Mathematics, Grenoble Institute of Technology, France}
\affil[4]{Stellenbosch Institute for Advanced Study (STIAS), Wallenberg Research Centre at Stellenbosch University, Stellenbosch, South Africa}

\date{}

\maketitle

\begin{abstract}
    The Synthetic Minority Oversampling TEchnique (SMOTE) is widely-used for the analysis of imbalanced datasets. It is known that SMOTE frequently over-generalizes the minority class, leading to misclassifications for the majority class, and effecting the overall balance of the model. 
    In this article, we present an approach that overcomes this limitation of SMOTE, employing Localized Random Affine Shadowsampling (LoRAS) to oversample from an approximated data manifold of the minority class. 
    We benchmarked our algorithm with 14 publicly available imbalanced datasets using three different Machine Learning (ML) algorithms and compared the performance of LoRAS, SMOTE and several SMOTE extensions that share the concept of using convex combinations of minority class data points for oversampling with LoRAS. We observed that LoRAS, on average generates better ML models in terms of F1-Score and Balanced accuracy. Another key observation is that while most of the extensions of SMOTE we have tested, improve the F1-Score with respect to SMOTE on an average, they compromise on the Balanced accuracy of a classification model. LoRAS on the contrary, improves both F1 Score and the Balanced accuracy thus produces better classification models.
    Moreover, to explain the success of the algorithm, we have constructed a mathematical framework to prove that LoRAS oversampling technique provides a better estimate for the mean of the underlying local data distribution of the minority class data space.
  
\end{abstract} 
\section{Introduction}\label{Intro}
 
Imbalanced datasets are frequent occurrences in a large spectrum of fields, where Machine Learning (ML) has found its applications, including business, finance and banking as well as bio-medical science. 
Oversampling approaches are a popular choice to deal with imbalanced datasets \citep{SMOTE, Han2, He, Bunkhumpornpat2009, Barua2014}. 
We here present Localized Randomized Affine Shadowsampling (LoRAS), which produces better ML models for imbalanced datasets, compared to state-of-the art oversampling techniques such as SMOTE and several of its extensions. 
We use computational analyses and a mathematical proof to demonstrate that drawing samples from a locally approximated data manifold of the minority class can produce balanced classification ML models.
We validated the approach with 12 publicly available imbalanced datasets, comparing the performances of several state-of-the-art convex-combination based oversampling techniques with LoRAS. The average performance of LoRAS on all these datasets is better than other oversampling techniques that we investigated. In addition, we have constructed a mathematical framework to prove that LoRAS is a more effective oversampling technique since it provides a better estimate for local mean of the underlying data distribution, in some neighbourhood of the minority class data space.\par 
For imbalanced datasets, the number of instances in one (or more) class(es) is very high (or very low) compared to the other class(es). A class having a large number of instances is called a majority class and one having far fewer instances is called a minority class. This makes it difficult to learn from such datasets using standard ML approaches. Oversampling approaches are often used to counter this problem by generating synthetic samples for the minority class to balance the number of data points for each class. SMOTE is a widely used oversampling technique, which has received various extensions since it was published by \cite{SMOTE}. The key idea behind SMOTE is to randomly sample artificial minority class data points along line segments joining the minority class data points among $k$ of the minority class nearest neighbors of some arbitrary minority class data point. In other words, SMOTE produces oversamples by generating random convex combinations of two close enough data points.\par 
The SMOTE algorithm, however has several limitations for example: it does not consider the distribution of minority classes and latent noise in a data set \citep{Hu2009}. It is known that SMOTE frequently over-generalizes the minority class, leading to misclassifications for the majority class, and effecting the overall balance of the model \citep{punt}. Several other limitations of SMOTE are mentioned in \cite{Blagus2013}.
To overcome such limitations, several algorithms have been proposed as extensions of SMOTE. Some are focusing on improving the generation of synthetic data by combining SMOTE with other oversampling techniques, including the combination of SMOTE with Tomek-links \citep{ElhassanT2016}, particle swarm optimization \citep{Gao, Wang}, rough set theory \citep{Ram}, kernel based approaches \citep{Mathew}, Boosting \citep{Chawla2}, and Bagging \citep{Hanifah}. Other approaches choose subsets of the minority class data to generate SMOTE samples or cleverly limit the number of synthetic data generated \citep{Narayan}. Some examples are Borderline1/2 SMOTE \citep{Han2},  ADAptive SYNthetic (ADASYN) \citep{He}, Safe Level SMOTE \citep{Bunkhumpornpat2009}, Majority Weighted Minority Oversampling TEchnique (MWMOTE) \citep{Barua2014}, Modified SMOTE (MSMOTE), and Support Vector Machine-SMOTE (SVM-SMOTE) \citep{Suh} (see Table \ref{table_samples_cf}) \citep{Hu2009}. Another recent method, G-SMOTE, generates synthetic samples  in  a  geometric  region  of  the  input  space, around  each  selected  minority  instance \citep{GSMOTE}. Voronoi diagrams have also been used in recent research for improving classification tasks for imbalanced datasets. Because of properties inherent to Voronoi diagrams, a newly proposed algorithm V-synth identifies exclusive regions of feature space where it is ideal to create synthetic minority samples \citep{Vor, Vor2}.\par 

{\textbf{Related research and novelty:} A more recent trend in the research on imbalanced datasets is to generate synthetic samples, aiming to approximate the latent data manifold of the minority class data space. In \cite{Belli}, a general framework for manifold-based oversampling, especially for high dimensional datasets, is proposed for synthetic oversampling. The method has been successfully applied in \cite{Belli2} to deal with gamma-ray spectra classification. It produces a synthetic set $S$ of $n$ instances in the manifold-space by randomly sampling $n$ instances from the PCA-transformed reduced data space. In order to produce unique samples on the manifold, they apply i.i.d. additive Gaussian noise to each sampled instance prior to adding it to the synthetic set $S$, controlling the distribution of the noise through the Gaussian distribution parameters. The synthetic Gaussian instances are then mapped back to the feature space to produce the final synthetic samples \citep{Belli}. Another scheme, using auto-encoders to oversample from an approximated manifold, has also been discussed in \cite{Belli}. This approach selects random minority class samples by adding Gaussian noise to them, and using the auto-encoder framework first maps them non-orthogonally off the manifold and then maps them back orthogonally on the manifold \cite{Belli}. It remains unclear from this research how the approach would perform in terms of improving F1-Scores of imbalanced classification models as it focuses on relative improvement in the Area Under the (ROC) Curve (AUC) as a performance measure. According to \cite{Saito}, AUC of the Receiver Operating Characteristic Curve (ROC) curve might not be informative enough for imbalanced datasets. This issue has also been addressed in \cite{Davis}. Unlike the work of \cite{Belli} LoRAS relies on locally approximating the manifold by generating random convex combination of noisy minority class data points. Our oversampling strategy LoRAS, rather aims at improving the precision-recall balance (F1-Score) and class wise average accuracy (Balanced accuracy) of the ML models used. The F1-Score can measure how well the classification model handled the minority class classification, whereas Balanced accuracy provides us with a measure of how both majority and minority classes were handled by the classification model. Thus, these two measures together can give us a holistic understanding of a classifier performance on a dataset. \par  

Notably, in the pre-SMOTE era of research, related to oversampling there has been works aiming to enrich minority classes of imbalanced datasets by adding Gaussian noise \cite{noisy} and using the noisy data itself, as oversampled data. The strategy of generating oversamples with convex combinations of minority class samples is also well known, SMOTE itself being an example of such a strategy. Our oversampling strategy LoRAS leverages from a combination of these two strategies. Unlike \cite{noisy}, we generate Gaussian noise in small neighbourhoods around the minority class samples and create our final synthetic data with convex combinations of multiple noisy data points (shadowsamples) as opposed to SMOTE based strategies, that consider combination of only two minority class data points. Adding the shadowsamples allows LoRAS to produce a better estimate for local mean of the latent minority class data distribution.}\par

We also provide a mathematical framework to show that convex combinations of multiple shadowsamples can provide a proper estimate for the local mean of a neighbourhood in the minority class data space. To be specific, an LoRAS oversample is an unbiased estimator of the mean of the underlying local probability distribution, followed by a minority class sample (assuming that it is some random variable) such that the variance of this estimator is significantly less than that of a SMOTE generated oversample, which is also an unbiased estimator of the mean of the underlying local probability distribution, followed by a minority class sample. In addition to this, LoRAS provides an option of choosing the neighbourhood of a minority class data point by performing prior manifold learning over the minority class using t-Stochastic Neighbourhood Embedding (t-SNE) \citep{tsne}. t-SNE is a state-of the art algorithm used for dimension reduction maintaining the underlying manifold structure in a sense that, in a lower dimension t-SNE can cluster points, that are close enough in the latent high dimensional manifold. It uses a symmetric version of the cost function used for it's predecessor technique Stochastic Neighbourhood Embedding (SNE) and uses a Student-t distribution rather than a Gaussian to compute the similarity between two points in the low-dimensional space. t-SNE employs a heavy-tailed distribution in the low-dimensional space to alleviate both the crowding problem and the optimization problems of SNE \citep{tsne,sne}.   \par

Till date there are at least eighty five extension models built on SMOTE \citep{SMOTEVAR}. Considering a large number of benchmark datasets explored in our study, it was necessary to shortlist certain oversampling algorithms for a comparative study. We found quite a few studies that have applied or explored SMOTE and extension of SMOTE such as Borderline1/2 SMOTE models, ADASYN, and SVM-SMOTE \citep{Suh, Ah-Pine2016,Adisania,Chiama, Wang, Le}. Moreover all these oversampling strategies are focused on oversampling from the convex hull of small neighbourhoods in the minority class data space, a similarity that they share with our proposed approach. Considering these factors, we choose to focus on these five oversampling strategies for a comparative study with our oversampling technique LoRAS.

\begin{longtable}[ht!]{l l} 
\caption{Popular algorithms built on SMOTE.}
\label{table_samples_cf} \tabularnewline
\hline
Extension & Description\\ [0.5ex] 
\hline\hline
Borderline1/2 SMOTE \citep{Han2} & {Identifies borderline samples and applies SMOTE on them}  \\ 
\hline
ADASYN \citep{He} &  {Adaptively changes the weights of different minority samples} \\
\hline
SVM-SMOTE \citep{Suh}  & {Generates new minority samples near borderlines with SVM}  \\
\hline
Safe-Level-SMOTE \citep{Bunkhumpornpat2009} & {Generates data in areas that are completely safe}  \\
\hline
MWMOTE \citep{Barua2014} &  {Identifies and weighs ambiguous minority class samples} \\
\hline
\vspace{-4mm}
\end{longtable}   

\section{LoRAS: Localized Randomized Affine Shadowsampling}\label{LoRAS}

In this section we discuss our strategy to approximate the data manifold, given a dataset. A typical dataset for a supervised ML problem consists of a set of \textit{features} $F=\{f_1, f_2, \dots\}$, that are used to characterize patterns in the data and a set of \textit{labels} or ground truth. Ideally, the number of instances or samples should be significantly greater than the number of features. In order to maintain the mathematical rigor of our strategy we propose the following definition for a \textit{small dataset}. 

\begin{definition}\label{sd}
Consider a class or the whole dataset with $n$ samples and $|F|$ features. If $\log_{10}(\frac{n}{|F|})<1$, then we call the dataset, a \textit{small dataset}.  
\end{definition}

The LoRAS algorithm is designed to learn from a dataset by approximating the underlying data manifold. 
Assuming that $F$ is the best possible set of features to represent the data and all features are equally important, we can think of a data oversampling model to be a function $g:  \prod_{i=1}^{l} R^{|F|} \rightarrow  R^{|F|}$, that is, $g$ uses $l$ parent data points (each with $|F|$ features) to produce an oversampled data point in $R^{|F|}$. 

\begin{definition}\label{afcomb}
We define a \textit{random affine combination} of some arbitrary vectors as the affine linear combination of those vectors, such that the coefficients of the linear combination are chosen randomly. Formally, a vector $v$, $v=\alpha_1u_1+\dots+\alpha_nu_m$, is a random affine combination of vectors $u_1,\dots,u_m$, ($u_j\in R^{|F|}$) if $\alpha_1+\dots+\alpha_m=1$, $\alpha_j\in R^{+}$ and $\alpha_1,\dots,\alpha_m$ are the coefficients of the affine combination chosen randomly from a Dirichlet distribution.
\end{definition}
The simplest way of augmenting a data point would be to take the average (or random affine combination with positive coefficients as defined in  Definition \ref{afcomb}) of two data points as an augmented data point. But, when we have $|F|$ features, we can assume that the hypothetical manifold on which our data lies is $|F|$-dimensional. An $|F|$-dimensional manifold can be locally approximated by a collection of $(|F|\!\!-\!\!1)$-dimensional planes.\par
Given $|F|$ sample points we could exactly derive the equation of an unique $(|F|\!\!-\!\!1)$-dimensional plane containing these $|F|$ sample points. Note that, a small neighbourhood of a dataset can itself be considered as a small dataset. A small neighbourhood of $k$ points around a data point in a dataset, given sufficiently small $k$, satisfies Definition \ref{sd}, that is $k$ and $|F|$ satisfies, $\log_{10}(\frac{k}{|F|})<1$. Thus, considering $k$ to be sufficiently small we can assume that this small neighbourhood is a small dataset. To enrich this small dataset, we create \textit{shadow data points} or \textit{shadowsamples} from our $k$ parent data points in the minority class data point neighbourhood. Each shadow data point is generated by adding noise from a normal distribution, $\mathscr{N}(0, h(\sigma_f))$ for all features $f \in F$, where $h(\sigma_f)$ is some function of the sample variance $\sigma_f$ for the feature $f$. For each of the $k$ data points we can generate $m$ shadow data points such that, $k \times m  \gg |F|$. Now it is possible for us to choose $|F|$ shadow data points from the $k \times m$ shadow data points even if $k<|F|$. We choose $|F|$ shadow data points as follows: we first choose a random parent data point $p$ and then restrict the domain of choice to the shadowsamples generated by the parent data points in $N_k^p$. \par

 For high dimensional datasets, choosing k-nearest neighbours of data point using simple Euclidean, Manhattan or general Minkowski distance measures can be misleading in terms of approximating the latent data manifold. To avoid this, we propose to adopt a manifold learning based strategy. Before choosing the k-nearest neighbours of a data point, we perform a dimension reduction on the data points of the minority class using the well-known dimension reduction and manifold learning technique t-SNE \citep{tsne}. Once we have a two dimensional t-embedding of the minority class data, we choose the k-nearest neighbours of a particular data point consistent to its k-nearest neighbours (measured as per usual distance metrics) in the 2-dimensional t-SNE embedding of the minority class.\par

Once we choose our neighbourhood and generate the shadowsamples, we take a random affine combination with positive co-efficients (Convex combination) of the $|F|$ chosen shadowsamples to create one augmented Localized Random Affine Shadowsample or a LoRAS sample as defined in  Definition \ref{afcomb}. Considering the arbitrary low variance that we can choose for the Normal distribution from which we draw our shadowsamples, we assume that our shadowsamples lie in the latent data manifold itself. It is a practical assumption, considering the stochastic factors leading to small measurement errors. Now, there exists an unique $(|F|\! - \!1)$-dimensional plane, that contains the $|F|$ shadowsamples, which we assume to be an approximation of the latent data manifold in that small neighbourhood.
Thus, a LoRAS sample is an artificially generated sample drawn from an $(|F|\! - \!1)$-dimensional plane, which locally approximates the underlying hypothetical $|F|$-dimensional data manifold.  It is worth mentioning here, that the effective number of features in a dataset is often less than $|F|$. In high dimensional data there are often correlated features or features with low variance. Thus, for practical use of LoRAS one might consider generating convex combinations of effective number of features which might be less than $|F|$. \par

\begin{algorithm}[ht!]
    \SetArgSty{textnormal}
    \caption{Localized Random Affine Shadowsample (LoRAS) Oversampling}
    \label{algo}
    \SetKwInOut{In}{Inputs}
    \In{\newline
        \vspace{-4mm}
        \begin{flushleft}
        \begin{tabular}{ m{5em} m{10cm} }
            $\var{C\_maj}$: & Majority class parent data points \\ 
            $\var{C\_min}$: & Minority class parent data points \\ 
        \end{tabular}
        \end{flushleft}
    }
    \SetKwInOut{Parameter}{Parameters}
    \Parameter{\newline
        \vspace{-3mm}
        \begin{flushleft}
        \begin{tabular}{ m{5em} m{12cm} }
            $\var{k}$: & Number of nearest neighbors to be considered per parent data point\hfill\break (default value : $30$ if $|C_{\text{min}}| >= 100$, $5$ otherwise) \\ 
            $\var{\!|S\textsubscript{p}|}$:& Number of generated shadowsamples per parent data point\hfill\break \big(default value : $\max \big(\ceil*{\frac{2|F|}{\var{k}}},40 \big)$\big) \\
            $\var{L\textsubscript{\textsigma}}$: & List of standard deviations for normal distributions for adding noise to each feature\hfill\break (default value : $[0.005, \ldots, 0.005]$) \\
            $\var{N\textsubscript{aff}}$:& Number of shadow points to be chosen for a random affine combination\hfill\break (default value : $|F|$) \\
            $\var{N\textsubscript{gen}}$:& Number of generated LoRAS points for each nearest neighbors group\hfill\break \big(default value : $\frac{|C_{\text{maj}}|-|C_{\text{min}}|}{|C_{\text{min}}|}$\big)\\
            $\var{embedding}$:&Type of Embedding used to choose minority class neighbourhood (regular or  t-embedding) \hfill\break (default value : `regular' )\\
            $\var{perplexity}$:& Perplexity of t-embedding (applicable only if $\var{embedding}$=`t-embedding') \hfill\break (default value : 30) 
        \end{tabular}
        \end{flushleft}
    }
    \SetKwInput{Parameter}{Constraint}
    \Parameter{\newline
        \vspace{-5mm}
        \begin{flushleft}
        \begin{tabular}{l @{\hskip3pt}l}
            $\var{N\textsubscript{aff}} < k * \var{|S\textsubscript{p}|}$
        \end{tabular}
        \end{flushleft}
    }

    Initialize $\var{loras\_set}$ as an empty list\\
    \SetKwFor{For}{For}{do}{endfor}
    \SetKwRepeat{Repeat}{Repeat}{Until}
    \For{each minority class parent data point $\var{p}$ in $\var{C\_min}$}{
      $\var{neighborhood} \xleftarrow{}$ calculate $\var{k}$-nearest neighbors of $\var{p}$, as per selected $\var{Embedding}$ parameter and append $\var{p}$\\
        
        Initialize $\var{neighborhood\_shadow\_sample}$ as an empty list\\
        
        \For{each parent data point $\var{q}$ in $\var{neighborhood}$}{
            $\var{shadow\_points} \xleftarrow{}$ draw $\var{|S\textsubscript{p}|}$ shadowsamples for $\var{q}$ drawing noises from normal distributions with corresponding standard deviations $\var{L\textsubscript{\textsigma}}$ containing elements for every feature\\
            Append $\var{shadow\_points}$ to $\var{neighborhood\_shadow\_sample}$
        }
        
        \Repeat{$\var{N\textsubscript{gen}}$ resulting points are created}{
            $\var{selected\_points} \xleftarrow{}$ select $\var{N\textsubscript{aff}}$ random shadow points from $\var{neighborhood\_shadow\_sample}$\\
            $\var{affine\_weights} \xleftarrow{}$ create and normalize random weights for  $\var{selected\_points}$\\
            $\var{generated\_LoRAS\_sample\_point} \xleftarrow{} \var{selected\_points} \cdot \var{affine\_weights}$\\
            Append $\var{generated\_LoRAS\_sample\_point}$ to $\var{loras\_set}$
        }
        
    }
    Return resulting set of generated LoRAS data points as $\var{loras\_set}$
    
\end{algorithm}

 In this article, all imbalanced classification problems that we deal with are binary classification problems. For such a problem, there is a minority class  $C_{\text{min}}$ containing a relatively less number of samples compared to a majority class $C_{\text{maj}}$. We can thus consider the minority class as a small dataset and use the LoRAS algorithm to oversample.  For every data point $p$ we can denote a set of shadowsamples generated from $p$ as $S_p$. In practice, one can also choose $2 \leq N_{\text{aff}} \leq |F|$ shadowsamples for an affine combination and choose a desired number of oversampled points $N_{\text{gen}}$ to be generated using the algorithm. We can look at LoRAS as an oversampling algorithm as described in Algorithm~\ref{algo}.\par
The LoRAS algorithm thus described, can be used for oversampling of minority classes in case of highly imbalanced datasets. 
Note that the input variables for our algorithm are: number of nearest neighbors per sample $\var{k}$, number of generated shadow points per parent data point $\var{|S\textsubscript{p}|}$, list of standard deviations for normal distributions for adding noise to every feature and thus generating the shadowsamples $\var{L\textsubscript{\textsigma}}$, number of shadowsamples to be chosen for affine combinations $\var{N\textsubscript{aff}}$, number of generated points for each nearest neighbors group $\var{N\textsubscript{gen}}$ and embedding strategy $\var{embedding}$. There is a conditional input variable $\var{perplexity}$ which takes a positive numerical value if one chooses a t-embedding. The perplexity parameter of the t-SNE algorithm is quite crucial. The perplexity parameter can influence the t-Embedding calculated by the t-SNE algorithm. There have been several studies that address the issue on finding a right perplexity parameter for a given problem \citep{perp}. That is why, we recommend careful choice of this parameter in order to leverage more from our algorithm. Another important parameter of our algorithm is the $\var{N\textsubscript{aff}}$. For this parameters an ideal choice would be the number of effective features in a dataset since this number would be a reasonable approximation to the dimension of the underlying data manifold. One could employ a feature selection technique to find out a good estimate for this. A simple random grid search is also very helpful to get reasonably good estimates of these parameters. 
We have mentioned all the default values of the LoRAS parameters in Algorithm~\ref{algo}, showing the pseudocode for the LoRAS algorithm.  
As an output, our algorithm generates a LoRAS dataset for the oversampled minority class, which can be subsequently used to train a ML model.\par

\begin{figure}[ht!] 
 \includegraphics[scale=.8]{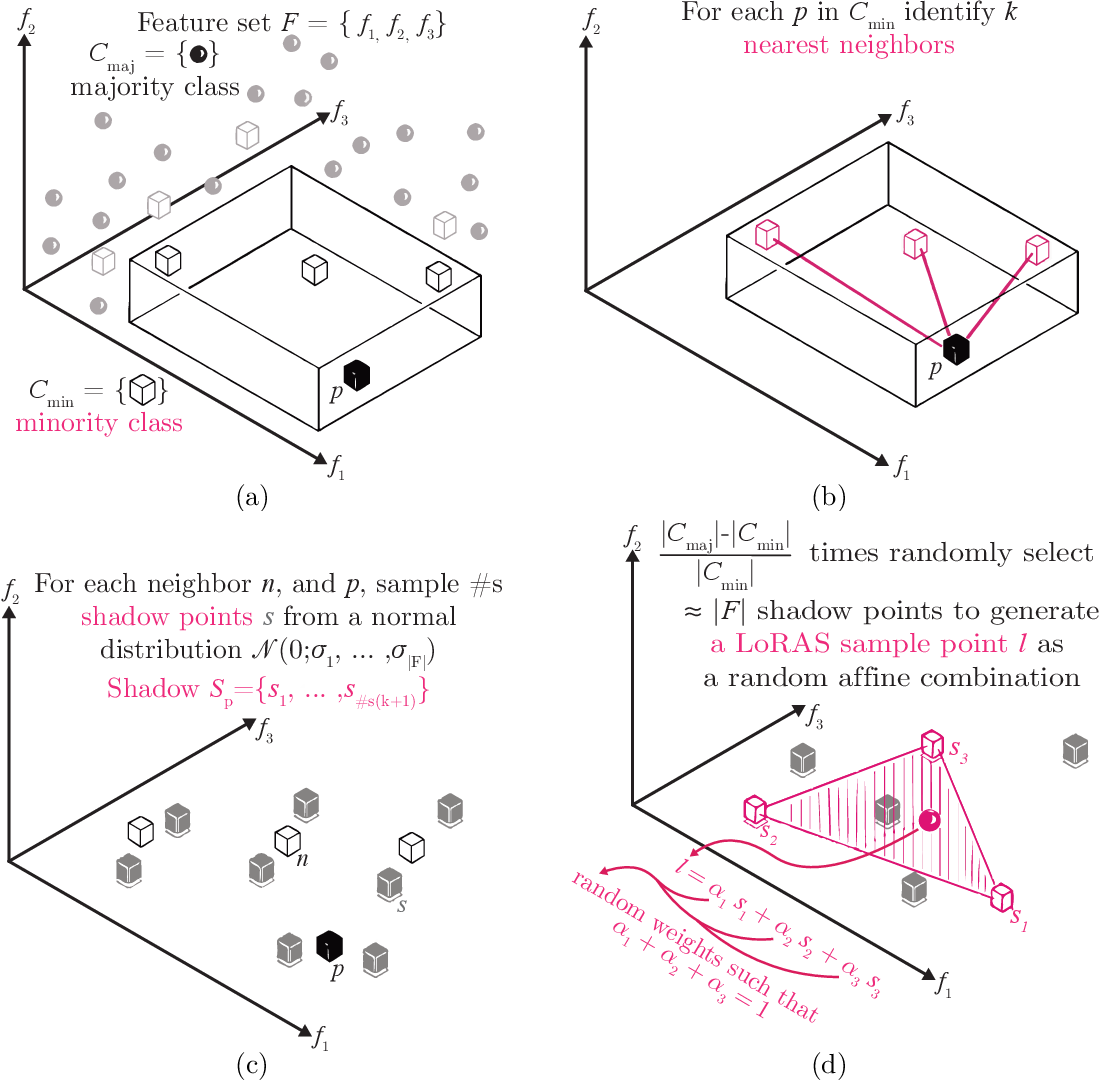}
    \caption{Visualization of the workflow demonstrating a step-by-step explanation for LoRAS oversampling. (a) Here, we show the parent data points of the minority class points $C_{\text{min}}$. For a data point $p$ we choose three of the closest neighbors (using knn) to build a neighborhood of $p$, depicted as the box. (b) Extracting the four data points in the closest neighborhood of $p$ (including $p$). (c) Drawing shadow points from a normal distribution, centered at these parent data point $n$. (d) We randomly choose three shadow points at a time to obtain a random affine combination of them (spanning a triangle). We finally generate a novel LoRAS sample point from the neighborhood of a single data point $p$.}
\end{figure}

\section{Case studies}\label{Case Studies and validation}

For testing the potential of LoRAS as an oversampling approach, we designed benchmarking experiments with a total of 14 datasets which are either highly imbalanced, high dimensional or with a small number of data points. With this number of diverse case studies we should have a comprehensive idea of the advantages of LoRAS over the other oversampling algorithms of our interest.  

\subsection{Datasets used for validation}\label{Datasets}
Here we provide a brief description of the datasets and the sources that we have used for our studies.\par  
\textbf{Scikit-learn imbalanced benchmark datasets:} The {\fontfamily{pcr}\selectfont imblearn.datasets} package is complementing the {\fontfamily{pcr}\selectfont sklearn.datasets} package. It provides 27 pre-processed datasets, which are imbalanced. The datasets span a large range of real-world problems from several fields such as business, computer science, biology, medicine, and technology. This collection of datasets was proposed in the {\fontfamily{pcr}\selectfont imblearn.datasets} python library by \cite{Lema} and benchmarked by \cite{Ding}. Many of these datasets have been used in various research articles on oversampling approaches \citep{Ding, saez}.  A statistically reliable benchmarking analysis of all 27 datasets in a stratified cross validation framework involves a lot of computational effort. We thus choose 11 datasets out of these two depending on two criteria: 
\begin{itemize}
    \item \textbf{Highly imbalanced:} We choose datasets with imbalance ratio more than 25:1. This category includes abalone\_19, letter\_image, mammography, ozone\_level, webpage, wine\_quality, yeast\_me2 datasets.  
    \item \textbf{High dimensional:} We choose the datasets with more than 100 features. This category includes arrhythmia, isolet, scene, webpage and yeast\_ml8.
\end{itemize}
Note that the $\var{webpage}$ dataset is common in both the criteria, giving us a total of 11 datasets. We choose these two categories because they are of special interest in research related to imbalanced datasets and have received extensive attention in this research area \citep{Anand, Hooda, Jing, Blagus2013}.\par
\textbf{Credit card fraud detection dataset:} We obtain the description of this dataset from the website. \url{https://www.kaggle.com/mlg-ulb/creditcardfraud}. \say{The dataset contains transactions made by credit cards in September 2013 by European cardholders. This dataset presents transactions that occurred in two days, where there are 492 frauds out of 284,807 transactions. The dataset is highly unbalanced, the positive class (frauds) account for 0.00172 percent of all transactions. The dataset contains only numerical input variables, which are the result of a PCA transformation. Feature variables $f_1, \dots, f_{28}$ are the principal components obtained with PCA, the only features that have not been transformed with PCA are the \q{Time} and \q{Amount}. The feature \q{Time} contains the seconds elapsed between each transaction and the first transaction in the dataset. The feature \q{Amount} consists of the transaction amount. The labels are encoded in the \q{Class} variable, which is the response variable and takes value 1 in case of fraud and 0 otherwise} \citep{cfraud}.\par
\textbf{Small datasets:} We were also interested to check the performance of LoRAS on small datasets. We obtained two such datasets: \href{https://www.openml.org/d/1059}{ar1}, \href{https://www.openml.org/d/1060}{ar3}. Both of these datasets have very few data points and less than 10 points in the minority class.\par
Thus, in total we benchmark our oversampling algorithms against the existing algorithms on a total of 14 datasets. We provide relevant statistics on these datasets in Table \ref{datastat}

\begin{table}[h!]
\caption{Table showing some statistics for the datasets we study in this article. For each dataset, we mark in bold the feature of the dataset that led us to its choice for our study.}

\label{datastat}
\begin{tabular}{l c c c}
\hline
Dataset & Imbalance ratio & Number of samples & Number of features \\
\hline\hline
abalone\_19 &	\textbf{130:1} & 4177 & 10  \\
\hline
arrythmia &	17:1 &	452 & \textbf{278} \\
\hline
isolet & 12:1 &	7797 & \textbf{617}   \\
\hline
letter-img & \textbf{26:1} & 20000 & 16 \\
\hline
mammography & \textbf{42:1} & 11183 & 6 \\
\hline
scene &	13:1 &	2407 &	\textbf{294}   \\
\hline
ozone\_level &	\textbf{34:1} &	2536 &	72 \\
\hline
webpage &	\textbf{33:1} &	34780 &	\textbf{300} \\
\hline
wine-quality  & \textbf{26:1}  &	4898  &	11    \\
\hline
yeast-me2  & \textbf{28:1}  &	1484  &	8  \\
\hline
yeast-ml8  & 13:1 &	2417  & \textbf{103}    \\
\hline
credit fraud  & \textbf{577:1} &	284807 &	28   \\
\hline
ar1  & 12.44:1 &	\textbf{121} &	30   \\
\hline
ar3  & 6.8:1 &	\textbf{63} &	30   \\
\hline
\end{tabular}
\end{table}

\subsection{Methodology}\label{method}
 For every dataset we have analyzed, we used a consistent workflow. Given a dataset, for every machine learning model, we judge the model performances based on a 5$\times$10-fold stratified cross validation framework. However, for the two small datasets ar1 and ar3 we use a 5$\times$3-fold stratified cross validation framework, since there are less than $10$ samples in the minority class. First we randomly scuffle the dataset. For a given dataset, we first split the dataset into 10 folds, each one distinct from the other maintaining the imbalance ratio for every fold. We then train the machine learning models on the dataset without any oversampling with 10-fold cross validation. This means that we train and test the model 10 times, each time considering a fold as a test fold and rest 9 folds as training folds. However, while training the ML models with oversampled data, we oversample only on the training folds and leave the test fold as they are for each training session. For each dataset we repeat the whole process five times to avoid the stochastic effects as much as possible.\par 

For the oversampling algorithms, for a given dataset, we chose the same neighbourhood size for every oversampling model. If there were less than 100 data points in the minority class the neighbourhood size was chosen to be 5. Otherwise we chose a neighbourhood size of 30. Given a large number of datasets we are analyzing, we did not customize this for every dataset and rather chose to stick to the above mentioned general rule. For LoRAS oversampling however, we performed a preliminary study to find out customized parameter values for every dataset, since the LoRAS algorithm is highly parametrized in nature. We tried several combinations of parameters $\var{N\textsubscript{aff}}$, $\var{embedding}$ and $\var{perplexity}$ employing random grid search.  For our initial study involving the parameter optimization of LoRAS, given a dataset, we performed a simple train-test split of the dataset (1:1 train-test split ratio), and then applied LoRAS with parameter grids on the training data to oversample and test the classifier performances on the test data. The training set is kept relatively small, so that the classifier does not gain much experience on the data while parameter estimation and gets prone to overfitting. This study was kept completely independent from our main cross-validation based results so that the samples from the test sets of our cross validation have minimum effect on parameter tuning. For parameter $\var{N\textsubscript{aff}}$ the grid interval is $[2,|F|]$, $|F|$ being the number of features. We choose five numbers while forming a search grid from this interval. Three of them are randomly chosen and the numbers $2$ and $|F|$ are always included in this set of $5$ numbers. For parameter $\var{embedding}$ we the grid values are the two possible entries that the parameter adopts. For the $\var{perplexity}$ parameter, we used grid values $[0.01, 0.1, 1, 10, 30, 100]$.\par
We emphasize here, that for all the algorithms including LoRAS, for a given dataset, we keep the neighbourhood size for every oversampling model fixed. For every oversampling model that we considered, the neighbourhood size for the oversampling model is the parameter that the model is highly sensitive to, since it contributes the most in determining the distribution of the oversampled minority class. For LoRAS, there are three (out of seven parameters in total)  parameters designed to better model/approximate the minority class data manifold (for example: the ones involving the t-SNE on the minority class), which are tuned to show the applicability of manifold approximation to improve convex combination based oversampling. However, as suggested, we keep all parameters related to the original distribution of the minority class, for all oversampling models fixed for all comparisons.
\par
However, considering the philosophy of LoRAS and a comparatively large number of parameters it use, we take liberty to tune the other parameters for LoRAS, since the other parameters are the key to a proper approximation or modelling of the minority class data manifold, which we argue to be the key factor behind the success of LoRAS.
\par
For LoRAS oversampling every dataset we use an unique value for $\var{N\textsubscript{aff}}$ as presented in Table \ref{paraset}. For individual ML models we use different settings for the LoRAS parameters $\var{embedding}$ and $\var{perplexity}$ which we mention explicitly in our supplementary materials while presenting the results for each ML model for each dataset. To ensure fairness of comparison, we oversampled such that the total number of augmented samples generated from the minority class was as close as possible to the number of samples in the majority class as allowed by each oversampling algorithm. Speaking of other parameters of the LoRAS algorithm, for $\var{L\textsubscript{\textsigma}}$, we chose a list consisting of a constant value of $.005$ for each dataset and for the parameter $\var{N\textsubscript{gen}}$ we chose the value as: $\frac{|C_{\text{maj}}|-|C_{\text{min}}|}{|C_{\text{min}}|}$. We provide a detailed list of parameter settings used by us for the oversampling algorithms in Table \ref{paraset}\par

\begin{longtable}[h!]{l |c| c| c }
\caption{In this table we present the details of parameter settings for the oversampling algorithms used by us for our experiment. The second column is the size of the oversampling neighbourhood and we have chosen the same size for all the oversampling models for each dataset in our analysis. The last three columns are specific to LoRAS parameters.}
\label{paraset} \tabularnewline
\hline
Dataset & Minority samples & Oversampling nbd & LoRAS $\var{N\textsubscript{aff}}$\\ [0.5ex] 
\hline\hline
abalone19 &	32 & 5 & 10  \\
\hline
arrythmia &	25 &	5 &	100\\
\hline
isolet & 600 &	30 &	179  \\
\hline
letter-img &	734 &	30 &	16 \\
\hline
mammography &	260 &	30 &	6  \\
\hline
scene &	177 &	30 &	2   \\
\hline
ozone\_level &	73 &	5 &	10 \\
\hline
webpage &	981 &	30 &	94 \\
\hline
wine-quality  & 183  &	30  &	2    \\
\hline
yeast-me2  &	51  &	5  &	2  \\
\hline
yeast-ml8  &	178 &	30  & 3    \\
\hline
credit fraud  & 492 &	30 &	30   \\
\hline
ar1  & 9 &	3 &	30   \\
\hline
ar3 & 8 &	3 &	10   \\
\hline
\end{longtable}

To choose ML models for our study we first did a pilot study with ML classifiers such as k-nearest neighbors (knn), Support Vector Machine (svm) (linear kernel), Logistic regression (lr), Random forest (rf), and Adaboost (ab). As inferred in \citep{Blagus2013} we found that knn was quite effective for the datasets we used. We also noticed that lr and svm performed better compared to rf and ab in most cases. We thus chose knn, svm and lr for our final studies. We used lbfgs solver for our logistic regression model and a linear kernel for our svm models. For our knn models, we choose 10 nearest neighbours for our prediction if there are less than 100 samples in the minority class and 30 nearest neighbours otherwise. For `arrhythmia', `abalone-19', `ar1' and `ar3' however we use only 5 nearest neighbours for the knn model since it has only  25, 32, 9 and 8 minority class samples respectively. We choose this parameter to be consistent to the neighbourhood size of the oversampling models, since the neighbourhood size directly influences the distribution of the training data and hence the model performance.\par

 In our analysis we take special notice of the credit card fraud detection dataset. This dataset is not included in the {\fontfamily{pcr}\selectfont imblearn.datasets} Python library. However, the main reason why we want to pay a special attention to this dataset is that, it is by far the most imbalanced publicly available dataset that we have come across. The extreme imbalance ratio of 577:1 is incomparable to any of the datasets in {\fontfamily{pcr}\selectfont imblearn.datasets}. Also, this dataset has received special attention of researchers attempting to use ML in Credit fraud detection \citep{credit}. In this article we see that lr and rf have good prediction accuracies on the dataset. Thus we chose these two ML models for the credit fraud dataset. \cite{credit} has also not provided cross validated analysis of their models, while our models have been trained and tested with the usual 10-fold cross validation framework as discussed before.
 Also, for two small datasets with a critically small minority class, we used only knn and lr classifiers, with parameter settings as specified before. The reason is, for all the 12 other datasets, svm did not stand out to be the best performer in terms of F1-Score in any of them.

For computational coding, we used the {\fontfamily{pcr}\selectfont scikit-learn (V 0.21.2)}, {\fontfamily{pcr}\selectfont  numpy (V 1.16.4)}, {\fontfamily{pcr}\selectfont  pandas (V 0.24.2)}, and {\fontfamily{pcr}\selectfont  matplotlib (V 3.1.0)} libraries in {\fontfamily{pcr}\selectfont  Python (V 3.7.4)}.\par

\section{Results}\label{results}

For imbalanced datasets there are more meaningful performance measures than \textit{Accuracy}, including \textit{Sensitivity} or \textit{Recall}, \textit{Precision}, and \textit{F1-Score} (\textit{F-Measure}), and \textit{Balanced accuracy} that can all be derived from the \textit{Confusion Matrix}, generated while testing the model. 
For a given class, the different combinations of recall and precision have the following meanings:
\begin{itemize}
    \item High Precision \& High Recall: The model handled the classification task properly
    \item High Precision \& Low Recall: The model cannot classify the data points of the particular class properly, but is highly reliable when it does so
    \item Low Precision \& High Recall: The model classifies the data points of the particular class well, but misclassifies high number of data points from other classes as the class in consideration
    \item Low Precision \& Low Recall: The model handled the classification task poorly
\end{itemize}
F1-Score, calculated as the harmonic mean of precision and recall and, therefore, balances a model in terms of precision and recall. 
These measures have been defined and discussed thoroughly by \cite{AbdElrahman2013}. 
Balanced accuracy is the mean of the individual class accuracies and in this context, it is more informative than the usual accuracy score. 
High Balanced accuracy ensures that the ML algorithm learns adequately for each individual class.\par
In our experiments we have noticed an interesting behaviour of oversampling models in terms of their average F1-Score and Balanced accuracy. Once we present our experiment results, we will discuss why considering F1-Score and Balanced accuracy can give us a clearer idea about model performances. 
We will use the above mentioned performance measures wherever applicable in this article.\par

\textbf{Selected model performances for all datasets:}
 We provide the detailed results of our experiments for all machine learning models as supplementary material. To be precise, for every combination of datasets, ML models and oversampling strategies we provide the mean and variance of the 10-fold cross validation process over 5 repetitions. For judging the performance of the oversampling models we follow the following scheme:
\begin{itemize}
    \item[-] First, for a given dataset, we choose the ML model trained on that dataset that provides the highest average F1-Score over all the oversampling models and training without oversampling. The F1-Score reflects the balance between precision and recall and considered as a reliable metric for imbalanced classification task.
    \item[-] We then consider the Balanced accuracy and F1- score of the chosen model as an evaluation of how well the oversampling model performs on the considered dataset. Following this evaluation scheme we present our results in Table \ref{table_imbsk}. 
\end{itemize}

\begin{longtable}[h!]{l |@{\hskip3pt}c@{\hskip3pt}|@{\hskip3pt} c @{\hskip3pt}|@{\hskip3pt} c@{\hskip3pt}|@{\hskip3pt} c@{\hskip3pt}| @{\hskip3pt}c @{\hskip3pt}|@{\hskip3pt}c@{\hskip3pt}|@{\hskip3pt} c@{\hskip3pt}|@{\hskip3pt} c@{\hskip3pt}}
\caption{Table showing Balanced accuracy/F1-Score for several oversampling strategies (Baseline, SMOTE, SVM-SMOTE, Borderline1 SMOTE, Borderline2 SMOTE, ADASYN and LoRAS column-wise respectively) for all 14 datasets of interest for ML learning models producing best average F1 score over all oversampling strategies and baseline training for respective datasets. }
\label{table_imbsk} \tabularnewline
\hline
Dataset & ML & Baseline & SMOTE & Bl-1 & Bl-2 & SVM & ADASYN & LoRAS\\ [0.5ex] 
\hline\hline

abalone19 &	knn &	.534/.000 &	.644/.054 &	.552/.044 &	.552/.044 &	.556/.045 &	.571/.055 &	\textbf{.675}/\textbf{.059} \\
\hline
arrythmia &	lr &	.679/.37 &	.666/.345 &	.672/.352 &	\textbf{.709}/.307 &	.679/.350 &	.667/.362 &	.694/\textbf{.380} \\
\hline
isolet &	lr &	.900/\textbf{.826} &	.898/.806 &	.899/.802 &	.906/.693 &	\textbf{.911}/.799 &	.898/.806 &	.904/.809 \\
\hline
letter-img &	knn &	.927/\textbf{.915} &	.988/.781 &	.984/.768 &	.977/.687 &	.986/.724 &	.985/.732 &	\textbf{.989}/.833\\
\hline
mammography &	knn &	.703/\textbf{.549} &	\textbf{.911}/.413 &	.909/.414 &	.899/.326 &	.909/.467 &	.905/.353 &	.896/.511 \\
\hline
scene &	lr &	.551/.168 &	.616/.222 &	.619/.230 &	\textbf{.620}/.223 &	.616/\textbf{.235} &	\textbf{.620}/.224 &	.616/.226 \\
\hline
ozone\_level &	lr &	.517/.062 &	.800/.190 &	.777/.212 &	.781/.183 &	.738/\textbf{.215} &	.803/.192 &	\textbf{.809}/.207 \\
\hline
webpage &	knn &	.805/\textbf{.711} &	.906/.267 &	.901/.274 &	.903/.287 &	.904/.267 &	.903/.264 &	\textbf{.923}/.613 \\
\hline
wine-quality  & lr  &	.517/.067  &	.718/.179  &	.715/.182  &	.711/.171  &	.712/\textbf{.216}  &	.721/.180  &	\textbf{.734}/.197 \\
\hline
yeast-ml8  &	knn  &	.500/.000  &	.558/.152  &	.561/.153  &	.563/.153  &	\textbf{.572}/\textbf{.158}  &	.558/.151  &	.559/.152 \\
\hline
yeast-me2  &	knn  &	.523/.074  &	.834/.331  &	.797/.373  &	.79/.304  &	.785/\textbf{.388}  & .825/.315	 &	\textbf{.842}/.354 \\
\hline
credit fraud  &	rf  &	.669/.775  &	.922/.359  &	.919/.645  &	.919/.556  &	.913/.741 &	.\textbf{.923}/.350  &	.904/\textbf{.820} \\
\hline
ar1  &	knn  &	.340/.071  &	.561/.306  &	.549/.298  &	\textbf{.594}/.338  &	.550/.324 &	.583/.320  &	.563/\textbf{.349} \\
\hline
ar3  &	rf  &	.634/.259  &	.810/.531  &	.809/\textbf{.584}  &	.819/.582  &	.755/.479 &	781/.457  &	\textbf{.823}/.563 \\
\hline
\textbf{Average} & - &	.636/.338 &	.775/.352 &	.764/.380 &	.771/.346 &	.759/.386 &	.777/.340 &	\textbf{.783}/\textbf{.433} \\
\hline
\textbf{Average rank} & - &	6.53/4.64 &	3.57/4.75 &	4.35/3.46 &	3.39/5.10 &	4.07/3.17 &	3.5/4.71 &	\textbf{2.57}/\textbf{2.14} \\
\hline
\end{longtable}

Calculating average performances over all datasets, LoRAS has the best Balanced accuracy and F1-Score. As expected, SMOTE improved Balanced accuracy compared to model training without any oversampling. Surprisingly, it lags behind in F1-Score,  for quite a few datasets with high baseline F1-Score such as letter\_image, isolet, mammography, webpage and credit fraud.    
Interestingly, the oversampling approaches SVM-SMOTE and Borderline1 SMOTE also improved the average F1-Score compared to SMOTE, but compromised for a lower Balanced accuracy. On the other hand, applying ADASYN increased the Balanced accuracy compared to SMOTE, but again compromises on the F1-Score. In contrast, our LoRAS approach produces the best Balanced accuracy on average by maintaining the highest average F1-Score among all oversampling techniques.  We want to emphasize that, even considering stochastic factors, LoRAS can improve both the Balanced accuracy and F1-Score of ML models significantly compared to SMOTE, which makes it unique.\par

\textbf{Datasets with high imbalance ratio:}
To verify the performance of LoRAS on highly imbalanced datasets we present average of the selected model performances for the datasets with highest imbalance ratios (among the ones we have tested) in Table \ref{imbr}.

\begin{longtable}[h!]{l |@{\hskip3pt}c@{\hskip3pt}|@{\hskip3pt} c @{\hskip3pt}|@{\hskip3pt} c@{\hskip3pt}|@{\hskip3pt} c@{\hskip3pt}| @{\hskip3pt}c @{\hskip3pt}|@{\hskip3pt}c@{\hskip3pt}|@{\hskip3pt} c@{\hskip3pt}}
\caption{ Table showing the average Balanced accuracy/F1-Score of the selected models for datasets with the highest imbalance ratios and high dimensional datasets separately}
\label{imbr} \tabularnewline
\hline
Average & Baseline & SMOTE & Bl-1 & Bl-2 & SVM & ADASYN & LoRAS\\ [0.5ex] 
\hline\hline
Highly imbalanced datasets &	.662/.381 &	.840/.321 &	.819/.364 &	.817/.319 &	.814/.382 &	.841/.305 &	\textbf{.846}/\textbf{.449} \\
\hline
High dimensional datasets & .687/.415 &	.728/.358 &	.730/.362 &	\textbf{.740}/.332 &	.736/.361 &	.729/.361 &	.739/\textbf{.436} \\
\hline
\end{longtable}
From our results we observe that LoRAS oversampling can significantly improve model performances for highly imbalanced datasets. LoRAS provides the highest F1-Score and Balanced accuracy among all the oversampling models. The results here show similar properties for SMOTE, Borderline-1 SMOTE, SVM SMOTE,  ADASYN and LoRAS as discussed before. Note that, for the credit fraud dataset, which is the most imbalanced among all, LoRAS has significant success over the other oversampling models in terms of Balanced accuracy. For the webpage dataset as well it improves the Balanced accuracy significantly, compromising minimally on the baseline F1-Score. The same trend follows for the letter\_image dataset. Notably, these three datasets have the highest number of overall samples as well, implying that with more data LoRAS can significantly outperform compared convex combination based oversampling models.
\par

\textbf{High dimensional datasets:} It is also of interest to us to check how LoRAS performs on high dimensional datasets. We therefore select five datasets with highest number of features among our tested datasets and present the performances of the selected ML methods in Table \ref{imbr}
From our results for high dimensional datasets, we observe that LoRAS produces the best F1-Score and second best Balanced accuracy on average among all oversampling models as Borderline-2 SMOTE beats LoRAS marginally. SMOTE improves both Balanced accuracy with respect to the baseline score here. Borderline-1 SMOTE and SVM SMOTE further increases SMOTE's performance both in terms of F1-Score and Balanced accuracy. Borderline-2 SMOTE, although improves the Balanced accuracy of SMOTE compromises on the F1-Score. Note that, even excluding the webpage dataset, where LoRAS has an overwhelming success, LoRAS still has the best average F1-Score and third highest Balanced accuracy marginally behind SVM-SMOTE and Borederline-2 SMOTE. We thus conclude, that for high dimensional datasets LoRAS can outperform the compared oversampling models in terms of F1-Score, while compromising marginally for Balanced accuracy.\par
\textbf{Small datasets:} For the two small datasets (with less than 10 samples in minority class) we have explored, we observed that LoRAS performs reasonably well. For the `ar1', LoRAS produces the best F1-Score and third best Balanced accuracy. For the `ar2' dataset LoRAS produces the best Balanced accuracy and the third best F1-Score. Note that LoRAS performs quite well for the `abalone' and `arrhythmia' datasets, which also have a small number of data points in the minority class.\par
 \textbf{Statistical Analysis:} Following \cite{Tarawneh}, we use the Wilcoxon's signed rank test to compare LoRAS against the other convex-combination based oversampling algorithms, in terms of both the performance measures we have used: F1-Score and Balanced accuracy. \cite{Tarawneh} chose this test for comparative studies since it is safer than parametric tests as it refrains from assuming homogeneity or normal distribution of data. Therefore, it can be applied to any classifier evaluation measure. \cite{Tarawneh} further confirms: `The Wilcoxon test aims to find if a null hypothesis is true or not. The null hypothesis $H_0$ assumes that there is no significant difference between the classification results (observations) obtained from two different methods. We assume that the null hypothesis is rejected if the p-value of the Wilcoxon test is less than $\alpha=0.05$'\citep{Tarawneh}.\par 

\begin{longtable}[h!]{l |@{\hskip3pt}c@{\hskip3pt}|@{\hskip3pt} c @{\hskip3pt}|@{\hskip3pt} c@{\hskip3pt}|@{\hskip3pt} c@{\hskip3pt}| @{\hskip3pt}c @{\hskip3pt}|@{\hskip3pt}c@{\hskip3pt}}
\caption{ Table showing p-values for comparison of LoRAS against the other oversampling algorithms, in terms of both the performance measures we have used: F1-Score and Balanced accuracy.}
\label{p-val wil} \tabularnewline
\hline
Measure & Baseline & SMOTE & Bl-1 & Bl-2 & SVM & ADASYN \\ [0.5ex] 
\hline\hline
F1-Score &	0.0303 &	0.0009 &	0.0479 &	.0035 &	0.0479 &	0.0009  \\
\hline
Balanced accuracy & 0.0009 &	0.0354 &	0.0258 &	0.5095 &	0.0382 &	0.1670  \\
\hline
\end{longtable}
From Table \ref{p-val wil} we observe that the p-values for all the paired tests are less than $0.05$ for the F1-Score, and therefore, the $H_0$ is rejected for all the paired tests in case of the F1-Score. Thus, the F1-Scores LoRAS produce have a big enough difference compared to the other compared algorithms, to be statistically significant. For Balanced accuracy, the algorithms Borderline-2 SMOTE and ADASYN do not show  significant statistical difference to LoRAS. However, since F1-Score is a more reliable and widely used metric for imbalanced datasets, we conclude that overall results generated by LoRAS are significantly different from the compared oversampling algorithms.\par
\cite{Tarawneh} also remarks that the p-value alone is informative enough and does not provide information about the relationship strength between variables. The p-values do not reveal whether the results are significantly different in favour of LoRAS or against LoRAS. For that following \cite{Tarawneh} we use the metrics $W_+$, $W_-$ and $R$. These are calculated using the following stpng:
\begin{itemize}
    \item For each data pair (involving LoRAS and some other oversampling algorithm) of model predictions , the difference between both predictions is calculated and stored in a vector $D$, excluding the zero difference values. 
    \item The signs of the difference is recorded in a sign vector $S$. 
    \item The entries in $|D|$ are ranked, forming a vector $R'$. In case of tied ranks, an average ranking scheme is adopted. This means, after ranking the entries of $|D|$ are ranked using integers and then, in case of tied entries the average of the integer ranks are considered as the average rank for all the respective tied entries with a specific tied value.
    \item Component-wise product of $S$ and $R'$ provides us with the vector $W$, the vector of the signed ranks. The sum of absolute values of the positive entries in $W$ is $W_+$ and the  sum of absolute values of the negative entries in $W$ is $W_-$. After this we define, $W_R=min\{W_+,W_-\}$
    \item Then the test statistic $Z$ is calculated by the equation
    \begin{equation}
    \begin{split}
    Z= \frac{W_R-\frac{n(n+1)}{4}}{\sqrt{\frac{n(n+1)(2n+1)}{24}-\frac{\Sigma t^3-\Sigma t}{48}}}
    \end{split}
    \end{equation}
    where $n$ is the number of components in $D$ and $t$ is the number of times some $i$-th entry occurs in $R'$, summed over all such repeated instances.
    \item Finally $R$ is calculated using $R=\frac{|Z|}{\sqrt{N}}$, where $N$ is the total number of datasets compared, which is $14$ in our case.
\end{itemize}

Note that a higher value $W_+$ for LoRAS indicates towards a superior performance of LoRAS and the value of $R$ indicates towards how superior(with a higher $W_+$)/ inferior(with a higher $W_-$) the performance of LoRAS is, compared to the other oversampling model for the tested datasets. \cite{Tarawneh} have considered ranges of $R\leq 0.1$, $0.1 < R \leq 0.5$ and $R>0.5$ to be indicators for small, medium and high degree of change (improvement or deterioration) in the predictive performance respectively.

\begin{longtable}[h!]{l |@{\hskip3pt}c@{\hskip3pt}|@{\hskip3pt} c @{\hskip3pt}|@{\hskip3pt} c@{\hskip3pt}|@{\hskip3pt} c@{\hskip3pt}| @{\hskip3pt}c @{\hskip3pt}|@{\hskip3pt}c@{\hskip3pt}}
\caption{ Table showing $W_+$/$W_-$ /$R$ for comparison of LoRAS against the other oversampling algorithms, in terms of both the performance measures we have used: F1-Score and Balanced accuracy.}
\label{wr wil} \tabularnewline
\hline
Measure & Baseline & SMOTE & Bl-1 & Bl-2 & SVM & ADASYN \\ [0.5ex] 
\hline\hline
F1-Score &	95/10/.713 &	105/0/.880 &	90/15/.629 &	102/3/.830 & 80/15/.629 &	105/0/.880  \\
\hline
Balanced accuracy & 105/0/.880 &	102/3/.830 &	95/10/.715 &	69/36/.286 &	95/10/.722 &	95/10/.837  \\
\hline
\end{longtable}

From Table \ref{wr wil} we note that, LoRAS has a higher $W_+$ value for both F1 Score and Balanced accuracy in comparison to each of the other convex combination based oversampling methods in consideration. Moreover for the F1 Score measure, the $R$ value is also more than $0.5$, indicating a high degree of improvement in F1-Score for LoRAS, over the considered oversampling models. Similarly, for Balanced accuracy, we find high degree of improvement for LoRAS, over all considered oversampling models except the Borderline-2 SMOTE, for which there is a medium degree of improvement. Overall, we thus conclude that LoRAS provides a significant improvement in performance over the compared convex combination based oversampling methods.   

\section{Discussion}\label{discussions}
We have constructed a mathematical framework to prove that LoRAS is a more effective oversampling technique since it provides a better estimate for the mean of the underlying local data distribution of the minority class data space. 
Let $X=(X_1,\dots,X_{|F|}) \in C_{\text{min}}$ be an arbitrary minority class sample. Let $N^X_k$ be the set of the k-nearest neighbors of $X$, which will consider the neighborhood of $X$. 
Both SMOTE and LoRAS focus on generating augmented samples within the neighborhood $N^X_k$ at a time. 
We assume that a random variable $X \in N^X_k$ follows a shifted t-distribution with $k$ degrees of freedom, location parameter $\mu$, and scaling parameter $\sigma$. 
Note that here $\sigma$ is not referring to the standard deviation but sets the overall scaling of the distribution \citep{Jackman}, which we choose to be the sample variance in the neighborhood of $X$. 
A shifted t-distribution is used to estimate population parameters, if there are less number of samples (usually, $\leq$ 30) and/or the population variance is unknown. 
Since in SMOTE or LoRAS we generate samples from a small neighborhood, we can argue in favour of our assumption that locally, a minority class sample $X$ as a random variable, follows a t-distribution. 
Following \cite{Blagus2013}, we assume that if $X, X'\in N^X_k$ then $X$ and $X'$ are independent. 
For $X, X'\in N^X_k$, we also assume:
\begin{equation}
\begin{split}
    \newcommand{\E}{{\rm I\kern-.3em E}}
    \E[X]&=\newcommand{\E}{{\rm I\kern-.3em E}}\E[X']\\
    & =\mu = (\mu_1,\dots,\mu_{|F|})\\
    \newcommand{\Var}{\mathrm{Var}}\Var[X]&=\newcommand{\Var}{\mathrm{Var}}\Var[X']\\
    & =\sigma^2\Big(\frac{k}{k-2}\Big) \\
    & =\sigma'^2=(\sigma'^2_1,\dots,\sigma'^2_{|F|})
\end{split}
\end{equation}  
where, $\newcommand{\E}{{\rm I\kern-.3em E}}
    \E[X]$ and $\newcommand{\Var}{\mathrm{Var}}\Var[X]$ denote the expectation and variance of the random variable $X$ respectively. However, the mean has to be estimated by an estimator statistic (i.e. a function of the samples). Both SMOTE and LoRAS can be considered as an estimator statistic for the mean of the t-distribution that $X \in C_{\text{min}}$ follows locally.
\begin{theorem}\label{thm}
Both SMOTE and LoRAS are unbiased estimators of the mean $\mu$ of the t-distribution that $X$ follows locally. However, the variance of the LoRAS estimator is less than the variance of SMOTE given that $|F|>2$.
\end{theorem}
\begin{proof}
A shadowsample $S$ is a random variable $S=X+B$ where $X \in N^X_k$, the neighborhood of some arbitrary $X \in C_{\text{min}}$ and $B$ follows $\mathscr{N}(0,\sigma_B)$.
\begin{equation}
\begin{split}
    \newcommand{\E}{{\rm I\kern-.3em E}}
    \E[S]&=\newcommand{\E}{{\rm I\kern-.3em E}}
    \E[X]+ \E[B]\\
    & =\mu \\
    \newcommand{\Var}{\mathrm{Var}}\Var[S]&= \newcommand{\Var}{\mathrm{Var}}\Var[X]+\Var[B]\\
    & =\sigma'^2+\sigma_B^2 
\end{split}
\end{equation}
assuming $S$ and $B$ are independent. Now, a LoRAS sample $L=\alpha_1S^1+\dots+\alpha_{|F|}S^{|F|}$, where $S^1,\dots,S^{|F|}$ are shadowsamples generated from the elements of the neighborhood of $X$, $N^X_k$, such that $\alpha_1+\dots+\alpha_{|F|}=1$. The affine combination coefficients $\alpha_1,\dots,\alpha_{|F|}$ follow a Dirichlet distribution with all concentration parameters assuming equal values of 1 (assuming all features to be equally important). For arbitrary $i,j \in \left\{1, \dots, |F| \right\}$,
\begin{align*}
    \newcommand{\E}{{\rm I\kern-.3em E}}
    \E[\alpha_i] &=\frac{1}{|F|}\\ \newcommand{\Var}{\mathrm{Var}}\Var[\alpha_i] &=\frac{|F|-1}{|F|^2(|F|+1)}\\
    \newcommand{\Cov}{\mathrm{Cov}}\Cov(\alpha_i,\alpha_j) &=\frac{-1}{|F|^2(|F|+1)}
\end{align*}
where $\newcommand{\Cov}{\mathrm{Cov}}\Cov(A,B)$ denotes the covariance of two random variables $A$ and $B$. Assuming $\alpha$ and $S$ to be independent,
\begin{equation}
\label{eq:expectation}
    \newcommand{\E}{{\rm I\kern-.3em E}}
    \E[L] =\E[\alpha_1]\E[S^1]+\dots+\E[\alpha_{|F|}]\E[S^{|F|}] = \mu
\end{equation}
Thus $L$ is an unbiased estimator of $\mu$.
For $j,k,l \in \left\{1, \dots, |F| \right\}$,
\begin{equation}
\begin{split}
    \newcommand{\Cov}{\mathrm{Cov}}\Cov[\alpha_kS^k_j,\alpha_lS^l_j] &  =\newcommand{\E}{{\rm I\kern-.3em E}}
    \E[\alpha_kS^k_j\alpha_lS^l_j]-\E[\alpha_kS^k_j]\E[\alpha_lS^l_j]\\
    & = \newcommand{\E}{{\rm I\kern-.3em E}}
    \E[\alpha_k\alpha_l]\mu_j^2-\frac{\mu_j^2}{|F|^2}\\
    & =\Big[\newcommand{\Cov}{\mathrm{Cov}}\Cov(\alpha_k,\alpha_l)+\frac{1}{|F|^2}\Big]\mu_j^2-\frac{\mu_j^2}{|F|^2}=\mu_j^2\Cov(\alpha_k,\alpha_l)
\end{split}
\end{equation}
since $\alpha_k\alpha_l$ is independent of $S^k_jS^l_j$. For an arbitrary $j$, $j$-th component of a LoRAS sample $L_j$
\begin{equation}
\begin{split}
    \newcommand{\Var}{\mathrm{Var}}\Var(L_j) &= \newcommand{\Var}{\mathrm{Var}}\Var(\alpha_1S^1_j+\dots+\alpha_{|F|}S^{|F|}_j)\\
    & = \newcommand{\Var}{\mathrm{Var}}\Var(\alpha_1S^1_j)+\dots+\Var(\alpha_{|F|}S^{|F|}_j)+\Sigma_{k=1}^{|F|}\Sigma_{l=1,l \neq k}^{|F|}\newcommand{\Cov}{\mathrm{Cov}}\Cov(\alpha_kS^k_j,\alpha_lS^l_j)\\
    & =\frac{\mu_j^2(|F|-1)+2(\sigma'^2_j+\sigma_{Bj}^2)|F|}{|F|(|F|+1)}-\frac{\mu_j^2(|F|-1)}{|F|(|F|+1)}\\
    & = \frac{2(\sigma'^2_j+\sigma_{Bj}^2)}{(|F|+1)} 
\end{split}
\end{equation}
For LoRAS, we take an affine combination of $|F|$ shadowsamples and SMOTE considers an affine combination of two minority class samples. Note, that since a SMOTE generated oversample can be interpreted as a random affine combination of two minority class samples, we can consider, $|F|=2$ for SMOTE, independent of the number of features. Also, from Equation \ref{eq:expectation}, this implies that SMOTE is an unbiased estimator of the mean of the local data distribution. Thus, the variance of a SMOTE generated sample as an estimator of $\mu$ would be $\frac{2\sigma'^2}{3}$ (since $B=0$ for SMOTE). But for LoRAS as an estimator of $\mu$, when $|F|>2$, the variance would be less than that of SMOTE. 
\end{proof}
This implies that, locally, LoRAS can estimate the mean of the underlying t-distribution better than SMOTE.
To visualize the key aspects of LoRAS oversampling, we provide the PCA plots for oversampled data from the ozone\_level dataset several oversampling methods we have studied in Figure \ref{PCA}.  
\begin{figure}[ht!] 
\includegraphics[scale=.725]{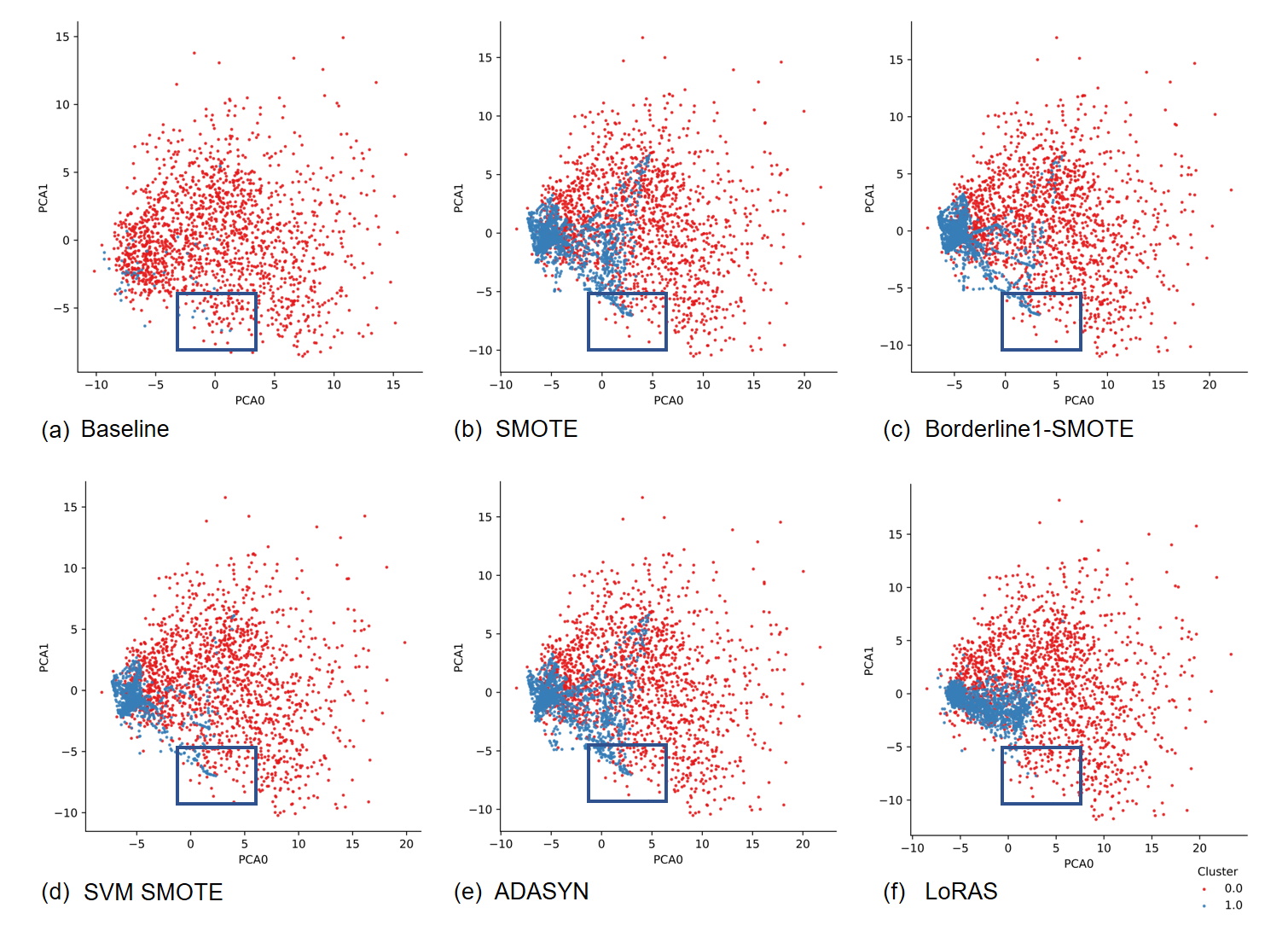}
\caption{Figure showing for Principal Component Analysis plot of ozone dataset for baseline data and oversampled data with several oversampling strategies for the ozone\_level dataset. The boxed region in each subplot shows a neighbourhood of outliers and how each oversampling stategy generates synthetic samples in that neighbourhood.}   \label{PCA}
\end{figure}
From Figure \ref{PCA} we can observe that SMOTE and ADASYN oversamples highly on the neighbourhood of the outliers, depicted by a blue box in each subplot. While this is somewhat controlled in Borderline1-SMOTE and SVM SMOTE, they still generate some synthetic samples in this neighbourhood. LoRAS on the other hand refrains, leveraging on its strategy to produce a better estimate for local mean of the underlying local data distribution. This enables LoRAS to ignore the outliers and to oversample more uniformly resulting in a better approximation of the data manifold. Note that, the average F1-Scores of the oversampling models as presented in Table \ref{table_imbsk} has a direct correlation to how the oversampling strategy oversamples in this neighbourhood. SMOTE and ADASYN generates the lowest F1-Scores and show a tendency of oversampling excessively from this neighbourhood. Borderline-SMOTE and SVM improves the F1-Score compared to SMOTE and ADASYN, again, consistent to their behaviour of oversampling lesser in this neighbourhood. LoRAS, has the highest average F1-Score and oversampling very sparsely from this neighbourhood.

\section{Conclusions}\label{Conclusion}

Oversampling with LoRAS produces comparatively balanced ML model performances on average, in terms of Balanced Accuracy and F1-Score among the compared convex-combination strategy based oversampling techniques. This is due to the fact that, in most cases LoRAS produces lesser mis-classifications on the majority class with a reasonably small compromise for mis-classifications on the minority class. From our study we infer that for tabular high dimensional and highly imbalanced datasets our LoRAS oversampling approach can better estimate the mean of the underlying local distribution for a minority class sample (considering it a random variable) and can improve Balanced accuracy and F1-Score of ML classification models. However, the scope of such convex combination based strategies including LoRAS, might be limited for heterogeneous image based imbalanced datasets.\par

The distribution of both the minority and majority class data points is considered in the oversampling techniques such as Borderline1 SMOTE, Borderline2 SMOTE, SVM-SMOTE, and ADASYN \citep{Gosain2017}. SMOTE and LoRAS are the only two techniques, among the oversampling techniques we explored, that deal with the problem of imbalance just by generating new data points, independent of the distribution of the majority class data points. Thus, comparing LoRAS and SMOTE gives a fair impression about the performance of our novel LoRAS algorithm as an oversampling technique, without considering any aspect of the distributions of the minority and majority class data points and relying just on resampling. Other extensions of SMOTE such as Borderline1 SMOTE, Borderline2 SMOTE, SVM-SMOTE, and ADASYN can also be built on the principle of LoRAS oversampling strategy. According to our analyses LoRAS already reveals great potential on a broad variety of applications and evolves as a true alternative to SMOTE, while processing highly unbalanced datasets.\par

\textbf{Availability of code:} A preliminary implementation of the algorithm in {\fontfamily{pcr}\selectfont  Python (V 3.7.4)} and an example {\fontfamily{pcr}\selectfont  Jupyter Notebook} for the credit card fraud detection dataset is provided on the GitHub repository \url{https://github.com/sbi-rostock/LoRAS}. This version does not yet include the t-embedding parameter.
In our computational code, 
$\var{\!|S\textsubscript{p}|}$ corresponds to $\var{num\_shadow\_points}$,  $\var{L\textsubscript{\textsigma}}$ corresponds to $\var{list\_sigma\_f}$, $\var{N\textsubscript{aff}}$ corresponds to $\var{num\_aff\_comb}$, $\var{N\textsubscript{gen}}$ corresponds to $\var{num\_generated\_points}$.

\textbf{Acknowledgements:} We thank Prof.\! Ria Baumgrass from Deutsches Rheuma-Forschungszentrum (DRFZ), Berlin for enlightening discussions on small datasets occuring in her research related to cancer therapy, that led us to the current work.
We thank the German Network for Bioinformatics Infrastructure (de.NBI) and  Establishment of Systems Medicine Consortium in Germany e:Med for their support, as well as the German Federal Ministry for Education and Research (BMBF) programs (FKZ 01ZX1709C) for funding us.

%
%
\bibliography{LoRAS.bib}

\section*{Supplementary data}
We provide the detailed individual results for each dataset and each ML model for our analysis as supplementary data. Here, we use the acronyms bl1, bl2, SVM and ADA for the oversampling models Borderline-1 SMOTE, Borderline-2 SMOTE, SVM-SMOTE and ADASYN respectively. We mark in bold for each dataset, the ML model with the highest average F1-Score, a criteria by which we select the model results to include in our further analysis.\\
\FloatBarrier
\textbf{Dataset: abalone\_19}
\begin{table}[h!]
    \centering
    \caption{F1-Scores for the logistic regression model for 5 runs of 10-fold cross validation for abalone\_19 dataset}
\begin{tabular}{p{30.5mm}|c|c|c|c|c|c|c}
\hline
Oversampling models &        lr1 &        lr2 &        lr3 &        lr4 &        lr5 &       mean &         sd \\
\hline
      Baseline &          0 &          0 &          0 &          0 &          0 &          0 &          0 \\
\hline
     SMOTE &      0.045 &      0.046 &      0.048 &      0.049 &      0.049 &     0.0474 & 0.00181659 \\
\hline
       bl1 &      0.048 &       0.05 &       0.04 &      0.051 &      0.043 &     0.0464 & 0.00472229 \\
\hline
       bl2 &      0.044 &      0.047 &       0.04 &      0.056 &      0.051 &     0.0476 &  0.0061887 \\
\hline
       SVM &      0.057 &      0.062 &      0.049 &      0.057 &      0.051 &     0.0552 & 0.00521536 \\
\hline
       ADA &      0.045 &      0.046 &      0.049 &       0.05 &      0.048 &     0.0476 & 0.00207364 \\
\hline
     LoRAS (Em=t,p=10) &      0.055 &      0.055 &      0.056 &      0.057 &      0.057 &      0.056 &      0.001 \\
\hline
     F1-Score average    &            &            &            &            &            & 0.04288571 &            \\
\hline
\end{tabular}
    
    \label{abalonef1lr}
\end{table}

\vspace{-4mm}
\begin{table}[h!]
    \centering
    \caption{Balanced accuracies for the lr model for 5 runs of 10-fold cross validation for abalone\_19 dataset}
    \begin{tabular}{p{36.5mm}|c|c|c|c|c|c|c}
\hline
   Oversampling model &    lr1 &        lr2 &        lr3 &        lr4 &        lr5 &       mean &         sd \\
\hline
    Baseline&   0.5 &        0.5 &      0.499 &      0.499 &      0.499 &     0.4994 & 0.00054772 \\
\hline
     SMOTE & 0.709 &       0.73 &      0.725 &      0.734 &      0.729 &     0.7254 & 0.00971082 \\
\hline
     bl1 & 0.677 &      0.686 &      0.623 &      0.696 &      0.659 &     0.6682 & 0.02870017 \\
\hline
     bl2 & 0.657 &      0.687 &       0.64 &      0.733 &      0.701 &     0.6836 & 0.03661694 \\
\hline
     SVM & 0.662 &      0.689 &      0.631 &      0.667 &      0.659 &     0.6616 & 0.02075572 \\
\hline
     ADA & 0.71 &      0.729 &      0.727 &      0.735 &      0.726 &     0.7254 & 0.00928978 \\
\hline
     LoRAS (Em=t,p=10) & 0.728 &      0.747 &      0.733 &      0.735 &      0.743 &     0.7372 & 0.00769415 \\
\hline
\end{tabular}  

    \label{abalonebalr}
\end{table}
\vspace{-4mm}
\begin{table}[h!]
    \centering
    \caption{F1-Scores for the svm model for 5 runs of 10-fold cross validation for abalone\_19 dataset}
    \begin{tabular}{p{30.5mm}|c|c|c|c|c|c|c}
\hline
    Oversampling models  &       svm1 &       svm2 &       svm3 &       svm4 &       svm5 &          mean  &     sd       \\
\hline
      Baseline &          0 &          0 &          0 &          0 &          0 &          0 &          0 \\
\hline
     SMOTE &      0.032 &       0.03 &      0.028 &      0.028 &      0.031 &     0.0298 & 0.00178885 \\
\hline
       bl1 &      0.023 &      0.025 &      0.024 &      0.027 &      0.027 &     0.0252 & 0.00178885 \\
\hline
       bl2 &      0.023 &      0.025 &      0.024 &      0.027 &      0.027 &     0.0252 & 0.00178885 \\
\hline
       SVM &      0.035 &      0.039 &      0.035 &      0.039 &      0.025 &     0.0346 & 0.00572713 \\
\hline
       ADA &      0.031 &       0.03 &      0.028 &      0.028 &      0.031 &     0.0296 & 0.00151658 \\
\hline
     LoRAS (Em=t,p=10)) &      0.032 &      0.033 &      0.029 &      0.033 &      0.032 &     0.0318 & 0.00164317 \\
\hline
     F1-Score average  &            &            &            &            &            & 0.02517143 &            \\
\hline
\end{tabular}  
    \label{abalonef1knn}
\end{table}
\vspace{-4mm}
\begin{table}[h!]
    \centering
    \caption{Balanced accuracies for the svm model for 5 runs of 10-fold cross validation for abalone\_19 dataset}
    \begin{tabular}{p{36.5mm}|c|c|c|c|c|c|c}
\hline
    Oversampling models &  svm1 &       svm2 &       svm3 &       svm4 &       svm5 &      mean      &     sd       \\
\hline
    Baseline &   0.5 &        0.5 &        0.5 &        0.5 &        0.5 &        0.5 &          0 \\
\hline
    SMOTE & 0.766 &      0.733 &      0.701 &      0.727 &       0.74 &     0.7334 & 0.02343715 \\
\hline
    bl1 & 0.615 &      0.662 &      0.635 &      0.703 &      0.684 &     0.6598 & 0.03563285 \\
\hline
    bl2 &  0.615 &      0.662 &      0.635 &      0.703 &      0.684 &     0.6598 & 0.03563285 \\
\hline
    SVM & 0.686 &      0.728 &      0.684 &      0.737 &      0.702 &     0.7074 & 0.02416195 \\
\hline
    ADA & 0.757 &      0.733 &      0.699 &      0.726 &      0.753 &     0.7336 & 0.02334095 \\
\hline
    LoRAS (Em=t,p=10)& 0.745 &      0.758 &      0.703 &      0.768 &      0.752 &     0.7452 & 0.02505394 \\
\hline
\end{tabular}  
    
\end{table}
\vspace{-4mm}
\begin{table}[h!]
    \centering
    \caption{F1-Scores for the knn model for 5 runs of 10-fold cross validation for abalone\_19 dataset}
    \begin{tabular}{p{30.5mm}|c|c|c|c|c|c|c}
\hline
       Oversampling models    &       knn1 &       knn2 &       knn3 &       knn4 &       knn5 &    mean        &     sd       \\
\hline
     Baseline &          0 &          0 &          0 &          0 &          0 &          0 &          0 \\
\hline
     SMOTE &      0.056 &       0.05 &      0.053 &      0.047 &      0.066 &     0.0544 & 0.00730068 \\
\hline
       bl1 &      0.044 &      0.059 &      0.035 &      0.052 &      0.034 &     0.0448 & 0.01080278 \\
\hline
       bl2 &      0.044 &      0.062 &      0.042 &      0.039 &      0.031 &     0.0436 &  0.0114149 \\
\hline
       SVM &      0.046 &      0.062 &      0.037 &      0.045 &      0.035 &      0.045 & 0.01065364 \\
\hline
       ADA &      0.057 &      0.049 &      0.053 &      0.054 &      0.063 &     0.0552 & 0.00521536 \\
\hline
     LoRAS (Em=r,p=NA)&      0.058 &      0.062 &      0.055 &      0.061 &      0.063 &     0.0598 & 0.00327109 \\
\hline
     F1-Score average   &            &            &            &            &            & \textbf{0.04325714} &            \\
\hline
\end{tabular}  
\end{table}
\vspace{-4mm}
\begin{table}[h!]
    \centering
    \caption{Balanced accuracies for the knn model for 5 runs of 10-fold cross validation for abalone\_19 datase}
    \begin{tabular}{p{36.5mm}|c|c|c|c|c|c|c}
\hline
     Oversampling models & knn1 &       knn2 &       knn3 &       knn4 &       knn5 &      mean      & sd           \\
\hline
     Baseline & 0.677 &        0.5 &        0.5 &        0.5 &        0.5 &     0.5354 & 0.07915681 \\
\hline
     SMOTE & 0.555 &      0.652 &      0.661 &      0.635 &       0.72 &     0.6446 & 0.05943316 \\
\hline
     bl1 & 0.572 &      0.569 &      0.522 &      0.564 &      0.537 &     0.5528 & 0.02210656 \\
\hline
     bl2 & 0.555 &      0.594 &      0.541 &      0.543 &      0.545 &     0.5556 & 0.02213143 \\
\hline
     SVM & 0.678 &       0.57 &      0.523 &      0.548 &      0.538 &     0.5714 & 0.06199032 \\
\hline
     ADA & 0.656 &      0.651 &      0.668 &      0.669 &       0.71 &     0.6708 & 0.02323144 \\
\hline
     LoRAS (Em=r,p=NA) & 0.666 &      0.683 &      0.664 &      0.667 &      0.698 &     0.6756 & 0.01463899 \\
\hline
\end{tabular}  
\end{table}
\FloatBarrier

\textbf{Dataset: arrhythmia}

\begin{table}[h!]
    \centering
    \caption{F1-Scores for the lr model for 5 runs of 10-fold cross validation for arrhythmia dataset}
    \begin{tabular}{p{30.5mm}|c|c|c|c|c|c|c}
\hline
        Oversampling models   &        lr1 &        lr2 &        lr3 &        lr4 &        lr5 &       mean &         sd \\
\hline
      Baseline &       0.36 &       0.39 &      0.354 &       0.35 &      0.396 &       0.37 & 0.02140093 \\
\hline
     SMOTE &      0.346 &       0.34 &      0.249 &       0.35 &       0.44 &      0.345 & 0.06762396 \\
\hline
       bl1 &      0.356 &      0.347 &       0.38 &      0.309 &      0.372 &     0.3528 &  0.0277074 \\
\hline
       bl2 &      0.334 &      0.362 &      0.317 &       0.27 &      0.252 &      0.307 & 0.04540925 \\
\hline
       SVM &      0.403 &      0.426 &      0.244 &      0.326 &      0.351 &       0.35 & 0.07141078 \\
\hline
       ADA &      0.398 &      0.306 &      0.336 &      0.336 &      0.435 &     0.3622 & 0.05270863 \\
\hline
     LoRAS (Em=t,p=1) &      0.429 &       0.43 &      0.288 &      0.353 &      0.403 &     0.3806 & 0.06045908 \\
\hline
    F1-Score average   &            &            &            &            &            & \textbf{0.35251429} &            \\
\hline
\end{tabular}  
\end{table}
\vspace{-4mm}
\begin{table}[h!]
    \centering
    \caption{Balanced accuracies for the lr model for 5 runs of 10-fold cross validation for arrhythmia dataset}
    \begin{tabular}{p{36.5mm}|c|c|c|c|c|c|c}
\hline
     Oversampling models & lr1 &        lr2 &        lr3 &        lr4 &        lr5 &       mean &         sd \\
\hline
     Baseline & 0.688 &      0.701 &      0.646 &      0.673 &      0.687 &      0.679 & 0.02094039 \\
\hline
     SMOTE & 0.674 &      0.673 &      0.624 &      0.673 &      0.686 &      0.666 & 0.02411431 \\
\hline
     bl1 &0.686 &      0.676 &      0.676 &      0.654 &      0.672 &     0.6728 & 0.01171324 \\
\hline
     bl2 & 0.729 &      0.766 &      0.697 &       0.68 &      0.673 &      0.709 & 0.03850325 \\
\hline
    SVM& 0.713 &      0.723 &      0.621 &      0.671 &      0.668 &     0.6792 & 0.04074555 \\
\hline
     ADA & 0.696 &      0.642 &      0.649 &      0.656 &      0.695 &     0.6676 & 0.02594802 \\
\hline
     LoRAS (Em=t,p=1) & 0.718 &      0.725 &      0.639 &      0.695 &      0.694 &     0.6942 & 0.03377425 \\
\hline
\end{tabular}  
\end{table}

\vspace{-4mm}

\begin{table}[h!]
    \centering
    \caption{F1-Scores for the svm model for 5 runs of 10-fold cross validation for arrhythmia dataset}
    \begin{tabular}{p{35.5mm}|c|c|c|c|c|c|c}
\hline
       Oversampling models    &       svm1 &       svm2 &       svm3 &       svm4 &       svm5 &       mean     &  sd          \\
\hline
      Baseline &      0.259 &      0.381 &      0.387 &      0.376 &      0.321 &     0.3448 & 0.05475582 \\
\hline
     SMOTE &      0.259 &      0.381 &      0.387 &      0.376 &      0.321 &     0.3448 & 0.05475582 \\
\hline
       bl1 &      0.259 &      0.381 &      0.387 &      0.376 &      0.321 &     0.3448 & 0.05475582 \\
\hline
       bl2 &      0.259 &      0.381 &      0.387 &      0.376 &      0.321 &     0.3448 & 0.05475582 \\
\hline
       SVM &      0.259 &      0.381 &      0.387 &      0.376 &      0.321 &     0.3448 & 0.05475582 \\
\hline
       ADA &      0.259 &      0.381 &      0.387 &      0.376 &      0.321 &     0.3448 & 0.05475582 \\
\hline
     LoRAS (Em=t,p=1)&      0.259 &      0.381 &      0.387 &      0.376 &      0.321 &     0.3448 & 0.05475582 \\
\hline
    F1-Score average       &            &            &            &            &            &     0.3448 &            \\
\hline
\end{tabular}  
\end{table}
\vspace{-4mm}
\begin{table}[h!]
    \centering
    \caption{Balanced accuracies for the svm model for 5 runs of 10-fold cross validation for arrhythmia dataset}
    \begin{tabular}{p{36.5mm}|c|c|c|c|c|c|c}
\hline
     Oversampling models & svm1 &       svm2 &       svm3 &       svm4 &       svm5 &      mean      &      sd     \\
\hline
     Baseline & 0.659 &      0.676 &      0.686 &      0.697 &      0.694 &     0.6824 & 0.01540454 \\
\hline
     SMOTE & 0.659 &      0.676 &      0.686 &      0.697 &      0.694 &     0.6824 & 0.01540454 \\
\hline
     bl1 & 0.659 &      0.676 &      0.686 &      0.697 &      0.694 &     0.6824 & 0.01540454 \\
\hline
     bl2 & 0.659 &      0.676 &      0.686 &      0.697 &      0.694 &     0.6824 & 0.01540454 \\
\hline
     SVM & 0.659 &      0.676 &      0.686 &      0.697 &      0.694 &     0.6824 & 0.01540454 \\
\hline
     ADA & 0.659 &      0.676 &      0.686 &      0.697 &      0.694 &     0.6824 & 0.01540454 \\
\hline
     LoRAS (Em=t,p=1)& 0.659 &      0.676 &      0.686 &      0.697 &      0.694 &     0.6824 & 0.01540454 \\
\hline
\end{tabular}  
\end{table}
\vspace{-4mm}
\begin{table}[h!]
    \centering
    \caption{F1-Scores for the knn model for 5 runs of 10-fold cross validation for arrhythmia dataset}
    \begin{tabular}{p{30.5mm}|c|c|c|c|c|c|c}
\hline
     Oversampling models      &       knn1 &       knn2 &       knn3 &       knn4 &       knn5 &    mean        &  sd          \\
\hline
      Baseline &          0 &          0 &          0 &          0 &        0.5 &        0.1 &  0.2236068 \\
\hline
     SMOTE &      0.185 &      0.191 &      0.202 &      0.192 &       0.22 &      0.198 & 0.01372953 \\
\hline
       bl1 &      0.187 &      0.195 &      0.182 &      0.194 &      0.158 &     0.1832 & 0.01505656 \\
\hline
       bl2 &      0.176 &      0.194 &      0.163 &      0.169 &      0.154 &     0.1712 & 0.01508973 \\
\hline
       SVM &      0.213 &      0.224 &      0.178 &      0.206 &      0.191 &     0.2024 & 0.01814663 \\
\hline
       ADA &      0.196 &      0.196 &      0.202 &      0.169 &      0.206 &     0.1938 & 0.01449828 \\
\hline
     LoRAS (Em=t,p=1)&      0.179 &      0.179 &      0.168 &      0.192 &      0.207 &      0.185 & 0.01494992 \\
\hline
      F1-Score average    &            &            &            &            &            & 0.17622857 &            \\
\hline
\end{tabular}  
\end{table}
\vspace{-4mm}

\begin{table}[h!]
    \centering
    \caption{Balanced accuracies for the knn model for 5 runs of 10-fold cross validation for arrhythmia dataset}
    \begin{tabular}{p{36.5mm}|c|c|c|c|c|c|c}
\hline
    Oversampling models &  knn1 &       knn2 &       knn3 &       knn4 &       knn5 &     mean       &    sd        \\
\hline
    Baseline &   0.5 &        0.5 &        0.5 &        0.5 &        0.5 &        0.5 &          0 \\
\hline
    SMOTE & 0.674 &      0.688 &       0.72 &      0.728 &      0.755 &      0.713 & 0.03234192 \\
\hline
    bl1 & 0.665 &      0.677 &       0.66 &      0.691 &      0.658 &     0.6702 & 0.01377316 \\
\hline
    bl2 & 0.666 &      0.689 &      0.642 &      0.666 &      0.658 &     0.6642 & 0.01697645 \\
\hline
    SVM & 0.684 &      0.681 &      0.609 &       0.69 &      0.689 &     0.6706 & 0.03463091 \\
\hline
    ADA & 0.712 &      0.693 &      0.719 &      0.727 &       0.73 &     0.7162 & 0.01475466 \\
\hline
    LoRAS (Em=t,p=1)& 0.664 &      0.673 &       0.67 &      0.728 &      0.715 &       0.69 & 0.02930017 \\
\hline
\end{tabular}  
\end{table}
\FloatBarrier

\textbf{Dataset: isolet}

\begin{table}[h!]
    \centering
    \caption{F1-Scores for the lr model for 5 runs of 10-fold cross validation for isolet dataset}
    \begin{tabular}{p{30.5mm}|c|c|c|c|c|c|c}
\hline
        Oversampling models   &        lr1 &        lr2 &        lr3 &        lr4 &        lr5 &       mean &         sd \\
\hline
      Base &      0.901 &      0.792 &      0.819 &      0.811 &      0.807 &      0.826 &  0.0430581 \\
\hline
     SMOTE &      0.799 &      0.798 &      0.808 &      0.817 &      0.809 &     0.8062 & 0.00785493 \\
\hline
       bl1 &      0.792 &      0.797 &      0.805 &      0.812 &      0.804 &      0.802 & 0.00771362 \\
\hline
       bl2 &      0.695 &      0.698 &      0.694 &        0.7 &      0.681 &     0.6936 &  0.0074364 \\
\hline
       SVM &      0.796 &      0.795 &      0.813 &      0.795 &        0.8 &     0.7998 & 0.00766159 \\
\hline
       ADA &      0.806 &      0.793 &      0.815 &      0.813 &      0.803 &      0.806 & 0.00877496 \\
\hline
     LoRAS (Em=t,p=30) &      0.809 &      0.793 &      0.821 &      0.816 &       0.81 &     0.8098 & 0.01056882 \\
\hline
     F1-Score average      &            &            &            &            &            & \textbf{0.79191429} &            \\
\hline
\end{tabular}  
\end{table}
\vspace{-4mm}
\begin{table}[h!]
    \centering
    \caption{Balanced accuracies for the lr model for 5 runs of 10-fold cross validation for isolet dataset}
    \begin{tabular}{p{36.5mm}|c|c|c|c|c|c|c}
\hline
    Ovesampling models &  lr1 &        lr2 &        lr3 &        lr4 &        lr5 &       mean &    sd        \\
\hline
     baseline & 0.901 &       0.89 &      0.906 &      0.903 &      0.903 &     0.9006 &  0.0061887 \\
\hline
     SMOTE & 0.894 &      0.892 &      0.901 &      0.904 &      0.902 &     0.8986 & 0.00527257 \\
\hline
     bl1 & 0.892 &      0.894 &      0.903 &      0.906 &        0.9 &      0.899 & 0.00591608 \\
\hline
     bl2 & 0.906 &      0.907 &      0.905 &      0.912 &      0.902 &     0.9064 & 0.00364692 \\
\hline
     SVM & 0.907 &      0.908 &       0.92 &      0.909 &      0.911 &      0.911 & 0.00524404 \\
\hline
     ADA & 0.899 &      0.891 &      0.901 &      0.904 &      0.897 &     0.8984 & 0.00487852 \\
\hline
     LoRAS (Em=t,p=30)& 0.905 &      0.894 &      0.909 &      0.908 &      0.906 &     0.9044 & 0.00602495 \\
\hline
\end{tabular}  
\end{table}
\vspace{-4mm}
\begin{table}[h!]
    \centering
    \caption{F1-Scores for the svm model for 5 runs of 10-fold cross validation for isolet dataset}
    \begin{tabular}{p{30.5mm}|c|c|c|c|c|c|c}
\hline
      Oversampling models     &       svm1 &       svm2 &       svm3 &       svm4 &       svm5 &     mean       &   variance         \\
\hline
      Base &      0.791 &      0.788 &      0.797 &      0.798 &      0.793 &     0.7934 & 0.00415933 \\
\hline
     SMOTE &      0.457 &      0.459 &      0.458 &      0.463 &       0.46 &     0.4594 & 0.00230217 \\
\hline
       bl1 &       0.49 &      0.489 &      0.486 &      0.494 &      0.498 &     0.4914 & 0.00466905 \\
\hline
       bl2 &       0.49 &      0.489 &      0.486 &      0.494 &      0.498 &     0.4914 & 0.00466905 \\
\hline
       SVM &      0.476 &      0.479 &      0.473 &      0.474 &      0.475 &     0.4754 & 0.00230217 \\
\hline
       ADA &      0.483 &      0.477 &       0.48 &      0.483 &      0.486 &     0.4818 & 0.00342053 \\
\hline
     LoRAS (Em=t,p=30)&      0.505 &      0.501 &      0.501 &      0.486 &      0.507 &        0.5 & 0.00824621 \\
\hline
     F1-Score average      &            &            &            &            &            & 0.52754286 &            \\
\hline
\end{tabular}  
\end{table}

\vspace{-4mm}

\begin{table}[h!]
    \centering
    \caption{Balanced Accuracies for the svm model for 5 runs of 10-fold cross validation for isolet dataset}
    \begin{tabular}{p{36.5mm}|c|c|c|c|c|c|c}
\hline
    Oversampling models &  svm1 &       svm2 &       svm3 &       svm4 &       svm5 &       mean     &      sd      \\
\hline
     Baseline & 0.893 &      0.893 &      0.906 &        0.9 &      0.904 &     0.8992 & 0.00605805 \\
\hline
     SMOTE & 0.894 &      0.893 &      0.895 &      0.897 &      0.895 &     0.8948 & 0.00148324 \\
\hline
     bl1 & 0.905 &      0.902 &      0.904 &      0.906 &      0.908 &      0.905 & 0.00223607 \\
\hline
     bl2 & 0.905 &      0.902 &      0.904 &      0.906 &      0.908 &      0.905 & 0.00223607 \\
\hline
     SVM & 0.899 &      0.901 &      0.899 &      0.899 &      0.897 &      0.899 & 0.00141421 \\
\hline
     ADA & 0.903 &      0.898 &      0.902 &      0.903 &      0.904 &      0.902 & 0.00234521 \\
\hline
     LoRAS (Em=t,p=30)& 0.91 &      0.906 &      0.909 &      0.908 &      0.911 &     0.9088 & 0.00192354 \\
\hline
\end{tabular}  
\end{table}

\vspace{-4mm}

\begin{table}[h!]
    \centering
    \caption{F1-Scores for the knn model for 5 runs of 10-fold cross validation for isolet dataset}
    \begin{tabular}{p{30.5mm}|c|c|c|c|c|c|c}
\hline
        Oversampling models   &       knn1 &       knn2 &       knn3 &       knn4 &       knn5 &    mean        & sd           \\
\hline\hline
      Base &      0.874 &      0.881 &       0.87 &      0.873 &      0.871 &     0.8738 & 0.00432435 \\
\hline
     SMOTE &      0.496 &      0.496 &      0.498 &      0.492 &      0.495 &     0.4954 & 0.00219089 \\
\hline
       bl1 &      0.506 &      0.507 &      0.508 &      0.508 &      0.508 &     0.5074 & 0.00089443 \\
\hline
       bl2 &       0.46 &      0.464 &      0.463 &      0.462 &      0.468 &     0.4634 & 0.00296648 \\
\hline
       SVM &      0.525 &      0.526 &      0.527 &      0.528 &       0.53 &     0.5272 & 0.00192354 \\
\hline
       ADA &      0.486 &      0.487 &      0.488 &      0.485 &      0.487 &     0.4866 & 0.00114018 \\
\hline
     LoRAS (Em=t,p=30)&      0.421 &      0.424 &       0.42 &      0.422 &      0.422 &     0.4218 & 0.00148324 \\
\hline

      F1-Score average    &            &            &            &            &            & 0.53937143 &            \\
\hline
\end{tabular}  
\end{table}
\vspace{-4mm}
\begin{table}[h!]
    \centering
    \caption{Balanced accuracies for the knn model for 5 runs of 10-fold cross validation for isolet dataset}
    \begin{tabular}{p{36.5mm}|c|c|c|c|c|c|c}
\hline
      Oversampling models & knn1 &       knn2 &       knn3 &       knn4 &       knn5 &   mean         & sd           \\
\hline
     Baseline & 0.914 &       0.92 &       0.91 &      0.915 &      0.913 &     0.9144 & 0.00364692 \\
\hline
     SMOTE & 0.915 &      0.915 &      0.915 &      0.913 &      0.913 &     0.9142 & 0.00109545 \\
\hline
     bl1 & 0.917 &      0.918 &      0.918 &      0.918 &      0.918 &     0.9178 & 0.00044721 \\
\hline
     bl2 & 0.905 &      0.903 &      0.902 &      0.902 &      0.904 &     0.9032 & 0.00130384 \\
\hline
     SVM & 0.923 &      0.923 &      0.924 &      0.924 &      0.925 &     0.9238 & 0.00083666 \\
\hline
     ADA & 0.911 &      0.912 &      0.912 &      0.911 &      0.911 &     0.9114 & 0.00054772 \\
\hline
     LoRAS (Em=t,p=30)& 0.883 &      0.884 &      0.882 &      0.884 &      0.883 &     0.8832 & 0.00083666 \\
\hline
\end{tabular}  
    
\end{table}
\FloatBarrier
\textbf{Dataset: letter\_image}
\begin{table}[h!]
    \centering
    \caption{F1-Scores for the lr model for 5 runs of 10-fold cross validation for letter\_image dataset}
  \begin{tabular}{p{30.5mm}|c|c|c|c|c|c|c}
\hline
       Oversampling models    &        lr1 &        lr2 &        lr3 &        lr4 &        lr5 &       mean &         sd \\
\hline
      Baseline &      0.743 &      0.745 &      0.745 &      0.744 &      0.742 &     0.7438 & 0.00130384 \\
\hline
     SMOTE &      0.581 &      0.587 &      0.582 &      0.585 &      0.583 &     0.5836 & 0.00240832 \\
\hline
       bl1 &      0.477 &      0.499 &      0.497 &      0.488 &      0.487 &     0.4896 & 0.00882043 \\
\hline
       bl2 &      0.331 &      0.337 &      0.337 &      0.339 &      0.339 &     0.3366 & 0.00328634 \\
\hline
       SVM &      0.463 &      0.468 &       0.47 &      0.467 &      0.468 &     0.4672 & 0.00258844 \\
\hline
       ADA &       0.52 &       0.52 &      0.521 &      0.522 &      0.518 &     0.5202 & 0.00148324 \\
\hline
     LoRAS(Em=r,p=NA) &      0.666 &      0.643 &      0.636 &      0.638 &      0.644 &     0.6454 & 0.01199166 \\
\hline
      F1-Score average     &            &            &            &            &            & 0.54091429 &            \\
\hline
\end{tabular}  

\end{table}
\vspace{-4mm}

\begin{table}[h!]
    \centering
    \caption{Balanced accuracies for the lr model for 5 runs of 10-fold cross validation for letter\_image dataset}
    \begin{tabular}{p{36.5mm}|c|c|c|c|c|c|c}
\hline
    Oversampling models &   lr1 &        lr2 &        lr3 &        lr4 &        lr5 &       mean &         sd \\
\hline
    Baseline & 0.831 &      0.833 &      0.833 &      0.831 &       0.83 &     0.8316 & 0.00134164 \\
\hline
    SMOTE & 0.953 &      0.956 &      0.954 &      0.955 &      0.953 &     0.9542 & 0.00130384 \\
\hline
    bl1 & 0.928 &      0.921 &      0.926 &      0.926 &       0.92 &     0.9242 & 0.00349285 \\
\hline
    bl2 & 0.912 &      0.907 &      0.905 &       0.91 &      0.906 &      0.908 & 0.00291548 \\
\hline
    SVM & 0.945 &      0.947 &      0.945 &      0.944 &      0.944 &      0.945 & 0.00122474 \\
\hline
    ADA & 0.948 &      0.949 &      0.948 &      0.949 &      0.948 &     0.9484 & 0.00054772 \\
\hline
    LoRAS (Em=r,p=NA)& 0.915 &      0.942 &      0.942 &      0.944 &      0.946 &     0.9378 & 0.01285302 \\
\hline
\end{tabular}  
\end{table}
\vspace{-4mm}
\begin{table}[h!]
    \centering
    \caption{F1-Scores for the svm model for 5 runs of 10-fold cross validation for letter\_image dataset}
    \begin{tabular}{p{30.5mm}|c|c|c|c|c|c|c}
\hline
       Oversampling models    &       svm1 &       svm2 &       svm3 &       svm4 &       svm5 &      mean      &      sd      \\
\hline
      Baseline &      0.753 &      0.757 &      0.757 &      0.762 &      0.761 &      0.758 & 0.00360555 \\
\hline
     SMOTE &      0.166 &      0.166 &      0.165 &      0.165 &      0.165 &     0.1654 & 0.00054772 \\
\hline
       bl1 &      0.174 &      0.173 &      0.174 &      0.174 &      0.174 &     0.1738 & 0.00044721 \\
\hline
       bl2 &      0.174 &      0.169 &      0.174 &      0.174 &      0.174 &      0.173 & 0.00223607 \\
\hline
       SVM &      0.168 &      0.177 &      0.169 &      0.169 &       0.17 &     0.1706 & 0.00364692 \\
\hline
       ADA &      0.177 &      0.186 &      0.177 &      0.177 &      0.177 &     0.1788 & 0.00402492 \\
\hline
     LoRAS (Em=t,p=1) &      0.204 &      0.206 &      0.204 &      0.204 &      0.203 &     0.2042 & 0.00109545 \\
\hline
       F1-Score average    &            &            &            &            &            & 0.26054286 &            \\
\hline
\end{tabular}  
\end{table}
\vspace{-4mm}
\begin{table}[h!]
    \centering
    \caption{Balanced accuracies for the svm model for 5 runs of 10-fold cross validation for letter\_image dataset}
    \begin{tabular}{p{36.5mm}|c|c|c|c|c|c|c}
\hline
       Oversampling models    &       svm1 &       svm2 &       svm3 &       svm4 &       svm5 &       mean &         sd \\
\hline
      Base &      0.829 &      0.831 &      0.831 &      0.833 &      0.832 &     0.8312 & 0.00148324 \\
\hline
     SMOTE &      0.809 &      0.809 &      0.808 &      0.807 &      0.807 &      0.808 &      0.001 \\
\hline
       bl1 &      0.819 &      0.819 &      0.819 &      0.819 &      0.819 &      0.819 &          0 \\
\hline
       bl2 &      0.819 &      0.819 &      0.819 &      0.819 &      0.819 &      0.819 &          0 \\
\hline
       SVM &      0.812 &      0.813 &      0.813 &      0.812 &      0.814 &     0.8128 & 0.00083666 \\
\hline
       ADA &      0.823 &      0.823 &      0.823 &      0.823 &      0.823 &      0.823 & 1.2413E-16 \\
\hline
     LoRAS (Em=t,p=1) &      0.851 &      0.852 &      0.851 &       0.85 &      0.849 &     0.8506 & 0.00114018 \\
\hline
\end{tabular}  
\end{table}
\vspace{-4mm}

\begin{table}[h!]
    \centering
    \caption{F1-Scores for the knn model for 5 runs of 10-fold cross validation for letter\_image dataset}
    \begin{tabular}{p{30.5mm}|c|c|c|c|c|c|c}
\hline
      Oversampling models     &       knn1 &       knn2 &       knn3 &       knn4 &       knn5 &      mean      &     sd       \\
\hline
      Base &      0.913 &      0.914 &      0.914 &      0.916 &      0.919 &     0.9152 & 0.00238747 \\
\hline\hline
     SMOTE &       0.78 &      0.782 &      0.782 &      0.782 &      0.782 &     0.7816 & 0.00089443 \\
\hline
       bl1 &      0.757 &      0.769 &      0.773 &      0.769 &      0.775 &     0.7686 &  0.0069857 \\
\hline
       bl2 &      0.757 &      0.675 &      0.659 &      0.674 &      0.672 &     0.6874 & 0.03943729 \\
\hline
       SVM &      0.662 &      0.741 &      0.738 &      0.741 &      0.738 &      0.724 &  0.0346915 \\
\hline
       ADA &      0.732 &      0.729 &      0.732 &      0.735 &      0.734 &     0.7324 & 0.00230217 \\
\hline
     LoRAS (Em=r,p=NA) &      0.839 &      0.831 &      0.832 &       0.83 &      0.833 &      0.833 & 0.00353553 \\
\hline
       F1-Score average    &            &            &            &            &            & \textbf{0.77745714} &            \\
\hline
\end{tabular}  
\end{table}
\vspace{-4mm}

\begin{table}[h!]
    \centering
    \caption{Balanced accuracies for the knn model for 5 runs of 10-fold cross validation for letter\_image dataset}
   \begin{tabular}{p{36.5mm}|c|c|c|c|c|c|c}
\hline
       Oversampling models    &       knn1 &       knn2 &       knn3 &       knn4 &       knn5 &       mean &         sd \\
\hline
      Base &      0.926 &      0.927 &      0.926 &      0.928 &       0.93 &     0.9274 & 0.00167332 \\
\hline
     SMOTE &      0.989 &      0.989 &      0.988 &      0.989 &      0.989 &     0.9888 & 0.00044721 \\
\hline
       bl1 &      0.984 &      0.983 &      0.984 &      0.986 &      0.985 &     0.9844 & 0.00114018 \\
\hline
       bl2 &      0.978 &      0.977 &      0.973 &       0.98 &      0.977 &      0.977 & 0.00254951 \\
\hline
       SVM &      0.986 &      0.986 &      0.986 &      0.986 &      0.986 &      0.986 &          0 \\
\hline
       ADA &      0.985 &      0.985 &      0.986 &      0.986 &      0.986 &     0.9856 & 0.00054772 \\
\hline
     LoRAS (Em=r,p=NA)&       0.99 &      0.989 &      0.989 &       0.99 &      0.989 &     0.9894 & 0.00054772 \\
\hline
\end{tabular}  

\end{table}
\FloatBarrier
\textbf{Dataset: mammography}

\begin{table}[h!]
    \centering
    
    \caption{F1-Scores for the lr model for 5 runs of 10-fold cross validation for mammography dataset}
    \begin{tabular}{p{30.5mm}|c|c|c|c|c|c|c}
\hline
Oversampling models &        lr1 &        lr2 &        lr3 &        lr4 &        lr5 &       mean &         sd \\
\hline
      Base &      0.532 &      0.535 &      0.524 &      0.533 &      0.533 &     0.5314 & 0.00427785 \\
\hline
     SMOTE &      0.283 &      0.286 &      0.282 &      0.282 &      0.288 &     0.2842 & 0.00268328 \\
\hline
       bl1 &      0.244 &      0.245 &      0.243 &      0.243 &      0.244 &     0.2438 & 0.00083666 \\
\hline
       bl2 &      0.218 &      0.217 &      0.216 &      0.216 &      0.217 &     0.2168 & 0.00083666 \\
\hline
       SVM &       0.32 &      0.315 &      0.314 &      0.312 &      0.314 &      0.315 &      0.003 \\
\hline
       ADA &      0.207 &       0.21 &       0.21 &      0.209 &      0.209 &      0.209 & 0.00122474 \\
\hline
     LoRAS (Em=t,p=.01) &      0.366 &      0.362 &      0.355 &      0.363 &      0.366 &     0.3624 & 0.00450555 \\
\hline
F1-Score average &            &            &            &            &            & 0.30894286 &            \\
\hline
\end{tabular}  

\end{table}

\vspace{-4mm}

\begin{table}[h!]
    \centering
    \caption{Balanced accuracies for the lr model for 5 runs of 10-fold cross validation for mammography dataset}
    \begin{tabular}{p{36.5mm}|c|c|c|c|c|c|c}
\hline
Oversampling models &        lr1 &        lr2 &        lr3 &        lr4 &        lr5 &       mean &         sd \\
\hline
      Base &      0.702 &      0.704 &        0.7 &      0.702 &      0.702 &      0.702 & 0.00141421 \\
\hline
     SMOTE &      0.885 &      0.884 &       0.88 &      0.881 &      0.886 &     0.8832 & 0.00258844 \\
\hline
       bl1 &      0.881 &      0.881 &      0.878 &      0.879 &       0.88 &     0.8798 & 0.00130384 \\
\hline
       bl2 &      0.872 &       0.87 &      0.868 &       0.87 &      0.872 &     0.8704 & 0.00167332 \\
\hline
       SVM &      0.883 &       0.88 &       0.88 &      0.882 &       0.88 &      0.881 & 0.00141421 \\
\hline
       ADA &      0.864 &      0.865 &      0.867 &      0.867 &      0.868 &     0.8662 & 0.00164317 \\
\hline
     LoRAS (Em=t,p=.01) &      0.853 &      0.853 &      0.863 &      0.856 &      0.859 &     0.8568 & 0.00426615 \\
\hline
\end{tabular}  

\end{table}
\vspace{-4mm}
\begin{table}[h!]
    \centering
    \caption{F1-Scores for the svm model for 5 runs of 10-fold cross validation for mammography dataset}
    \begin{tabular}{p{30.5mm}|c|c|c|c|c|c|c}
\hline
Oversampling models &       svm1 &       svm2 &       svm3 &       svm4 &       svm5 &       mean &         sd \\
\hline
      Base &      0.458 &      0.435 &      0.412 &      0.437 &      0.418 &      0.432 & 0.01806931 \\
\hline
     SMOTE &      0.097 &      0.096 &      0.097 &     0.0096 &      0.096 &    0.07912 & 0.03886608 \\
\hline
       bl1 &      0.098 &      0.102 &      0.103 &      0.103 &      0.103 &     0.1018 & 0.00216795 \\
\hline
       bl2 &      0.098 &      0.102 &      0.103 &      0.103 &      0.103 &     0.1018 & 0.00216795 \\
\hline
       SVM &        0.1 &      0.096 &      0.096 &      0.096 &      0.096 &     0.0968 & 0.00178885 \\
\hline
       ADA &        0.1 &      0.096 &      0.095 &      0.095 &      0.095 &     0.0962 & 0.00216795 \\
\hline
     LoRAS (Em=t,p=.01)&      0.108 &      0.106 &      0.108 &      0.106 &      0.106 &     0.1068 & 0.00109545 \\
\hline
F1-Score average &            &            &            &            &            & 0.14493143 &            \\
\hline
\end{tabular}  

\end{table}
\vspace{-4mm}
\begin{table}[h!]
    \centering
    \caption{Balanced accuracies for the svm model for 5 runs of 10-fold cross validation for mammography dataset}
    \begin{tabular}{p{36.5mm}|c|c|c|c|c|c|c}
\hline
Oversampling models &       svm1 &       svm2 &       svm3 &       svm4 &       svm5 &       mean &         sd \\
\hline
      Base &       0.66 &      0.649 &      0.637 &      0.647 &      0.639 &     0.6464 & 0.00915423 \\
\hline
     SMOTE &      0.751 &      0.754 &      0.759 &      0.756 &      0.756 &     0.7552 & 0.00294958 \\
\hline
       bl1 &      0.748 &      0.768 &      0.771 &      0.771 &      0.771 &     0.7658 & 0.01003494 \\
\hline
       bl2 &      0.748 &      0.768 &      0.771 &      0.771 &      0.771 &     0.7658 & 0.01003494 \\
\hline
       SVM &      0.763 &      0.754 &      0.756 &      0.755 &      0.753 &     0.7562 & 0.00396232 \\
\hline
       ADA &      0.764 &      0.755 &      0.754 &      0.753 &      0.753 &     0.7558 & 0.00465833 \\
\hline
     LoRAS (Em=t,p=.01)&       0.78 &      0.774 &       0.78 &      0.773 &      0.775 &     0.7764 & 0.00336155 \\
\hline
\end{tabular}  

\end{table}

\vspace{-4mm}

\begin{table}[h!]
    \centering
    \caption{F1-Scores for the knn model for 5 runs of 10-fold cross validation for mammography dataset}
    \begin{tabular}{p{30.5mm}|c|c|c|c|c|c|c}
\hline
Oversampling models &       knn1 &       knn2 &       knn3 &       knn4 &       knn5 &       mean &         sd \\
\hline
      Base &      0.557 &      0.554 &      0.545 &      0.544 &       0.54 &      0.548 & 0.00717635 \\
\hline
     SMOTE &      0.408 &      0.417 &      0.411 &      0.416 &      0.416 &     0.4136 & 0.00391152 \\
\hline
       bl1 &      0.417 &      0.413 &      0.416 &      0.417 &      0.411 &     0.4148 & 0.00268328 \\
\hline
       bl2 &      0.324 &       0.33 &      0.331 &      0.326 &      0.323 &     0.3268 & 0.00356371 \\
\hline
       SVM &      0.473 &      0.469 &      0.463 &      0.463 &      0.467 &      0.467 & 0.00424264 \\
\hline
       ADA &      0.356 &      0.354 &      0.353 &      0.352 &      0.354 &     0.3538 & 0.00148324 \\
\hline
     LoRAS (Em=r,p=NA)&      0.515 &      0.512 &      0.505 &      0.511 &      0.512 &      0.511 & 0.00367423 \\
\hline
F1-Score average &            &            &            &            &            & \textbf{0.43357143} &            \\
\hline
\end{tabular}  

\end{table}
\vspace{-4mm}
\begin{table}[h!]
    \centering
    \caption{Balanced accuracies for the knn model for 5 runs of 10-fold cross validation for mammography dataset}
    \begin{tabular}{p{36.5mm}|c|c|c|c|c|c|c}
\hline
Oversampling models &       knn1 &       knn2 &       knn3 &       knn4 &       knn5 &       mean &         sd \\
\hline
      Base &      0.706 &      0.706 &      0.701 &      0.702 &        0.7 &      0.703 & 0.00282843 \\
\hline
     SMOTE &      0.909 &      0.914 &      0.913 &      0.912 &       0.91 &     0.9116 & 0.00207364 \\
\hline
       bl1 &       0.91 &      0.908 &      0.912 &      0.908 &      0.908 &     0.9092 & 0.00178885 \\
\hline
       bl2 &      0.898 &        0.9 &      0.901 &        0.9 &      0.899 &     0.8996 & 0.00114018 \\
\hline
       SVM &      0.913 &       0.91 &       0.91 &      0.906 &      0.908 &     0.9094 & 0.00260768 \\
\hline
       ADA &      0.903 &       0.91 &      0.905 &      0.904 &      0.906 &     0.9056 & 0.00270185 \\
\hline
     LoRAS (Em=r,p=NA) &        0.9 &      0.894 &      0.899 &      0.896 &      0.894 &     0.8966 & 0.00279285 \\
\hline
\end{tabular}  
\end{table}

\FloatBarrier
\textbf{Dataset: scene}
\begin{table}[h!]
    \centering
    \caption{F1-Scores for the lr model for 5 runs of 10-fold cross validation for scene dataset}
\begin{tabular}{p{30.5mm}|c|c|c|c|c|c|c}
\hline
Oversampling models &        lr1 &        lr2 &        lr3 &        lr4 &        lr5 &       mean &         sd \\
\hline
      Base &      0.178 &      0.181 &      0.173 &      0.172 &      0.138 &     0.1684 & 0.01738678 \\
\hline
     SMOTE &      0.236 &      0.231 &      0.222 &      0.216 &      0.205 &      0.222 & 0.01226784 \\
\hline
       bl1 &      0.246 &      0.242 &      0.244 &      0.224 &      0.198 &     0.2308 & 0.02032732 \\
\hline
       bl2 &      0.227 &       0.24 &      0.226 &      0.221 &      0.205 &     0.2238 & 0.01263725 \\
\hline
       SVM &      0.234 &      0.243 &      0.237 &       0.24 &      0.224 &     0.2356 & 0.00730068 \\
\hline
       ADA &      0.231 &      0.242 &      0.233 &      0.216 &        0.2 &     0.2244 & 0.01653179 \\
\hline
     LoRAS (Em=t,p=30)&      0.239 &      0.234 &      0.233 &       0.22 &      0.206 &     0.2264 & 0.01339029 \\
\hline
F1-Score average &            &            &            &            &            & \textbf{0.21877143} &            \\
\hline
\end{tabular}  

\end{table}
\vspace{-4mm}

\begin{table}[h!]
    \centering
    \caption{Balanced accuracies for the lr model for 5 runs of 10-fold cross validation for scene dataset}
    \begin{tabular}{p{36.5mm}|c|c|c|c|c|c|c}
\hline
Oversampling models &        lr1 &        lr2 &        lr3 &        lr4 &        lr5 &       mean &         sd \\
\hline
      Base &      0.554 &      0.559 &      0.553 &      0.553 &      0.536 &      0.551 & 0.00874643 \\
\hline
     SMOTE &      0.628 &      0.629 &      0.613 &      0.612 &        0.6 &     0.6164 & 0.01217785 \\
\hline
       bl1 &      0.632 &      0.632 &      0.631 &      0.613 &      0.589 &     0.6194 & 0.01882286 \\
\hline
       bl2 &      0.625 &      0.638 &      0.621 &      0.616 &      0.603 &     0.6206 & 0.01277889 \\
\hline
       SVM &      0.611 &      0.622 &      0.616 &      0.618 &      0.607 &     0.6148 & 0.00589067 \\
\hline
       ADA &      0.625 &       0.64 &      0.626 &      0.613 &      0.596 &       0.62 & 0.01647726 \\
\hline
     LoRAS (Em=t,p=30)&      0.631 &      0.628 &      0.625 &       0.61 &        0.6 &     0.6188 & 0.01325519 \\
\hline
\end{tabular}  
\end{table}

\vspace{-4mm}

\begin{table}[h!]
    \centering
    \caption{F1-Scores for the svm model for 5 runs of 10-fold cross validation for scene dataset}
    \begin{tabular}{p{30.5mm}|c|c|c|c|c|c|c}
\hline
Oversampling models &       svm1 &       svm2 &       svm3 &       svm4 &       svm5 &       mean     &     sd       \\
\hline
      Base &      0.333 &      0.189 &      0.149 &      0.169 &      0.158 &     0.1996 &  0.0760513 \\
\hline
     SMOTE &      0.226 &      0.198 &      0.203 &      0.188 &      0.194 &     0.2018 & 0.01460137 \\
\hline
       bl1 &      0.222 &      0.214 &      0.204 &      0.199 &      0.193 &     0.2064 & 0.01163185 \\
\hline
       bl2 &      0.222 &      0.214 &      0.204 &      0.199 &      0.193 &     0.2064 & 0.01163185 \\
\hline
       SVM &      0.269 &      0.203 &      0.208 &      0.189 &      0.189 &     0.2116 & 0.03317831 \\
\hline
       ADA &      0.241 &      0.205 &      0.195 &      0.191 &      0.195 &     0.2054 &  0.0205621 \\
\hline
     LoRAS (Em=t,p=30)&      0.239 &      0.211 &      0.209 &      0.191 &      0.197 &     0.2094 & 0.01851486 \\
\hline
F1-Score average &            &            &            &            &            &     0.2058 &            \\
\hline
\end{tabular}  
\end{table}
\vspace{-4mm}

\begin{table}[h!]
    \centering
    \caption{Balanced accuracies for the svm model for 5 runs of 10-fold cross validation for scene dataset}
\begin{tabular}{p{36.5mm}|c|c|c|c|c|c|c}
\hline
Oversampling models &       svm1 &       svm2 &       svm3 &       svm4 &       svm5 &       mean &         sd \\
\hline
      Base &      0.507 &      0.516 &      0.539 &      0.552 &      0.545 &     0.5318 & 0.01935717 \\
\hline
     SMOTE &      0.633 &       0.63 &      0.639 &      0.613 &      0.624 &     0.6278 & 0.00988433 \\
\hline
       bl1 &      0.622 &      0.657 &       0.64 &      0.631 &      0.619 &     0.6338 & 0.01535252 \\
\hline
       bl2 &      0.622 &      0.657 &       0.64 &      0.631 &      0.619 &     0.6338 & 0.01535252 \\
\hline
       SVM &      0.644 &       0.64 &      0.652 &      0.617 &      0.615 &     0.6336 & 0.01665233 \\
\hline
       ADA &      0.649 &      0.643 &      0.624 &      0.619 &      0.624 &     0.6318 & 0.01329286 \\
\hline
     LoRAS (Em=t,p=30)&      0.639 &      0.651 &      0.645 &      0.616 &      0.627 &     0.6356 & 0.01409965 \\
\hline
\end{tabular}  
\end{table}
\vspace{-4mm}

\begin{table}[h!]
    \centering
    \caption{F1-Scores for the knn model for 5 runs of 10-fold cross validation for scene dataset}
\begin{tabular}{p{30.5mm}|c|c|c|c|c|c|c}
\hline
Oversampling models &       knn1 &       knn2 &       knn3 &       knn4 &       knn5 &       mean &         sd \\
\hline
      Base &      0.011 &          0 &      0.011 &          0 &      0.011 &     0.0066 & 0.00602495 \\
\hline
     SMOTE &      0.217 &      0.209 &      0.217 &      0.215 &      0.218 &     0.2152 & 0.00363318 \\
\hline
       bl1 &      0.232 &      0.235 &      0.234 &      0.238 &      0.236 &      0.235 & 0.00223607 \\
\hline
       bl2 &      0.234 &      0.235 &      0.238 &      0.233 &      0.232 &     0.2344 & 0.00230217 \\
\hline
       SVM &      0.247 &       0.26 &      0.259 &      0.248 &      0.263 &     0.2554 & 0.00736885 \\
\hline
       ADA &      0.208 &      0.211 &      0.213 &      0.214 &       0.21 &     0.2112 & 0.00238747 \\
\hline
     LoRAS (Em=t,p=30)&      0.222 &      0.222 &      0.223 &      0.225 &      0.223 &      0.223 & 0.00122474 \\
\hline
F1-Score average &            &            &            &            &            & 0.19725714 &            \\
\hline
\end{tabular}  
\end{table}

\vspace{-4mm}

\begin{table}[h!]
    \centering
    \caption{Balanced accuracies for the knn model for 5 runs of 10-fold cross validation for scene dataset}
    \begin{tabular}{p{36.5mm}|c|c|c|c|c|c|c}
\hline
Oversampling models &       knn1 &       knn2 &       knn3 &       knn4 &       knn5 &       mean &         sd \\
\hline
      Base &      0.502 &        0.5 &      0.502 &        0.5 &      0.502 &     0.5012 & 0.00109545 \\
\hline
     SMOTE &      0.698 &       0.68 &      0.701 &        0.7 &        0.7 &     0.6958 & 0.00889944 \\
\hline
       bl1 &       0.71 &      0.715 &      0.714 &      0.719 &      0.716 &     0.7148 & 0.00327109 \\
\hline
       bl2 &      0.711 &      0.717 &      0.722 &      0.714 &      0.711 &      0.715 & 0.00463681 \\
\hline
       SVM &      0.704 &      0.724 &      0.722 &      0.702 &      0.725 &     0.7154 & 0.01139298 \\
\hline
       ADA &      0.684 &       0.69 &      0.697 &      0.695 &       0.69 &     0.6912 & 0.00506952 \\
\hline
     LoRAS (Em=t,p=30)&        0.7 &      0.701 &      0.705 &       0.71 &      0.706 &     0.7044 & 0.00403733 \\
\hline
\end{tabular}  
\end{table}

\FloatBarrier

\textbf{Dataset: ozone\_level}

\begin{table}[h!]
    \centering
    \caption{F1-Scores for the lr model for 5 runs of 10-fold cross validation for ozone\_level dataset}
    \begin{tabular}{p{30.5mm}|c|c|c|c|c|c|c}
\hline
Oversampling models &        lr1 &        lr2 &        lr3 &        lr4 &        lr5 &       mean &         sd \\
\hline
      Base &      0.091 &      0.063 &      0.069 &      0.041 &       0.05 &     0.0628 & 0.01918854 \\
\hline
     SMOTE &      0.199 &      0.191 &       0.19 &      0.191 &       0.18 &     0.1902 & 0.00676018 \\
\hline
       bl1 &      0.218 &      0.202 &      0.219 &      0.213 &      0.211 &     0.2126 & 0.00680441 \\
\hline
       bl2 &       0.18 &      0.172 &      0.194 &      0.186 &      0.186 &     0.1836 & 0.00817313 \\
\hline
       SVM &      0.217 &      0.216 &      0.211 &      0.219 &      0.212 &      0.215 & 0.00339116 \\
\hline
       ADA &        0.2 &       0.19 &      0.191 &      0.197 &      0.186 &     0.1928 & 0.00563028 \\
\hline
     LoRAS (Em=t,p=10)&      0.216 &      0.204 &      0.209 &      0.205 &      0.205 &     0.2078 & 0.00496991 \\
\hline
F1-Score average &            &            &            &            &            & \textbf{0.18068571} &            \\
\hline
\end{tabular}  

\end{table}

\vspace{-4mm}
\begin{table}[h!]
    \centering
    \caption{Balanced accuracies for the lr model for 5 runs of 10-fold cross validation for ozone\_level dataset}
    \begin{tabular}{p{36.5mm}|c|c|c|c|c|c|c}
\hline
Oversampling models &        lr1 &        lr2 &        lr3 &        lr4 &        lr5 &       mean &         sd \\
\hline
      Base &      0.525 &       0.52 &       0.52 &      0.511 &      0.513 &     0.5178 & 0.00571839 \\
\hline
     SMOTE &      0.812 &      0.795 &      0.805 &      0.804 &      0.785 &     0.8002 & 0.01042593 \\
\hline
       bl1 &      0.785 &      0.771 &      0.778 &      0.775 &      0.776 &      0.777 & 0.00514782 \\
\hline
       bl2 &      0.767 &      0.768 &      0.795 &      0.795 &      0.781 &     0.7812 &   0.013755 \\
\hline
       SVM &      0.735 &      0.737 &      0.733 &      0.747 &      0.739 &     0.7382 &  0.0054037 \\
\hline
       ADA &      0.816 &        0.8 &        0.8 &      0.808 &      0.793 &     0.8034 & 0.00882043 \\
\hline
     LoRAS (Em=t,p=10)&      0.813 &      0.808 &      0.807 &      0.809 &      0.808 &      0.809 & 0.00234521 \\
\hline
\end{tabular}  

\end{table}
\vspace{-4mm}

\begin{table}[h!]
    \centering
    \caption{F1-Scores for the svm model for 5 runs of 10-fold cross validation for ozone\_level dataset}
    \begin{tabular}{p{30.5mm}|c|c|c|c|c|c|c}
\hline
Oversampling models &       svm1 &       svm2 &       svm3 &       svm4 &       svm5 &       mean &         sd \\
\hline
      Base &          0 &       0.02 &      0.022 &      0.018 &       0.02 &      0.016 & 0.00905539 \\
\hline
     SMOTE &      0.121 &      0.122 &      0.123 &      0.122 &      0.121 &     0.1218 & 0.00083666 \\
\hline
       bl1 &      0.131 &      0.135 &      0.142 &       0.13 &       0.14 &     0.1356 & 0.00531977 \\
\hline
       bl2 &      0.169 &      0.135 &      0.142 &       0.13 &       0.14 &     0.1432 & 0.01515586 \\
\hline
       SVM &      0.122 &      0.169 &      0.174 &      0.175 &      0.177 &     0.1634 & 0.02333024 \\
\hline
       ADA &      0.133 &      0.121 &      0.124 &      0.122 &      0.122 &     0.1244 &  0.0049295 \\
\hline
     LoRAS (Em=t,p=30)&      0.132 &      0.136 &      0.136 &      0.135 &      0.136 &      0.135 & 0.00173205 \\
\hline
F1-Score average &            &            &            &            &            & 0.11991429 &            \\
\hline
\end{tabular}  
\end{table}
\vspace{-4mm}

\begin{table}[h!]
    \centering
    \caption{Balanced accuracies for the svm model for 5 runs of 10-fold cross validation for ozone\_level dataset}
\begin{tabular}{p{36.5mm}|c|c|c|c|c|c|c}
\hline
Oversampling models &       svm1 &       svm2 &       svm3 &       svm4 &       svm5 &       mean &         sd \\
\hline
      Base &      0.498 &      0.505 &      0.506 &      0.503 &      0.504 &     0.5032 & 0.00311448 \\
\hline
     SMOTE &      0.753 &      0.756 &      0.758 &      0.748 &      0.754 &     0.7538 & 0.00376829 \\
\hline
       bl1 &      0.756 &      0.775 &      0.785 &      0.746 &      0.787 &     0.7698 & 0.01810249 \\
\hline
       bl2 &      0.756 &      0.775 &      0.785 &      0.746 &      0.787 &     0.7698 & 0.01810249 \\
\hline
       SVM &      0.787 &      0.791 &       0.79 &      0.792 &      0.803 &     0.7926 & 0.00610737 \\
\hline
       ADA &      0.755 &      0.751 &       0.76 &      0.749 &      0.759 &     0.7548 & 0.00481664 \\
\hline
     LoRAS (Em=t,p=30)&      0.778 &      0.782 &      0.782 &      0.774 &      0.788 &     0.7808 & 0.00521536 \\
\hline
\end{tabular}  

\end{table}

\vspace{-4mm}

\begin{table}[h!]
    \centering
    \caption{F1-Scores for the knn model for 5 runs of 10-fold cross validation for ozone\_level dataset}
\begin{tabular}{p{30.5mm}|c|c|c|c|c|c|c}
\hline
Oversampling models &       knn1 &       knn2 &       knn3 &       knn4 &       knn5 &       mean &         sd \\
\hline
      Base &          0 &          0 &          0 &          0 &          0 &          0 &          0 \\
\hline
     SMOTE &      0.095 &      0.098 &      0.117 &      0.112 &      0.129 &     0.1102 & 0.01398928 \\
\hline
       bl1 &      0.099 &      0.081 &      0.118 &      0.126 &      0.123 &     0.1094 & 0.01903418 \\
\hline
       bl2 &      0.108 &      0.089 &       0.11 &      0.107 &      0.105 &     0.1038 & 0.00846759 \\
\hline
       SVM &      0.102 &       0.09 &       0.11 &      0.101 &      0.127 &      0.106 & 0.01372953 \\
\hline
       ADA &      0.094 &      0.101 &      0.113 &       0.11 &       0.11 &     0.1056 & 0.00789303 \\
\hline
     LoRAS (Em=t,p=30)&      0.136 &      0.124 &      0.113 &      0.114 &      0.142 &     0.1258 & 0.01296919 \\
\hline
F1-Score average &            &            &            &            &            &     0.0944 &            \\
\hline
\end{tabular}  

\end{table}

\vspace{-4mm}
\begin{table}[h!]
    \centering
    \caption{Balanced accuracies for the knn model for 5 runs of 10-fold cross validation for ozone\_level dataset}
    \begin{tabular}{p{36.5mm}|c|c|c|c|c|c|c}
\hline
Oversampling models &       knn1 &       knn2 &       knn3 &       knn4 &       knn5 &       mean &         sd \\
\hline
      Base &        0.5 &        0.5 &        0.5 &        0.5 &        0.5 &        0.5 &          0 \\
\hline
     SMOTE &      0.598 &      0.603 &       0.64 &      0.633 &      0.664 &     0.6276 & 0.02733679 \\
\hline
       bl1 &      0.574 &      0.552 &      0.602 &      0.617 &      0.606 &     0.5902 & 0.02659323 \\
\hline
       bl2 &      0.616 &      0.584 &      0.621 &      0.621 &      0.614 &     0.6112 & 0.01551451 \\
\hline
       SVM &       0.57 &      0.557 &      0.578 &      0.572 &      0.603 &      0.576 & 0.01692631 \\
\hline
       ADA &      0.595 &       0.61 &      0.635 &      0.627 &      0.624 &     0.6182 &  0.0158019 \\
\hline
     LoRAS (Em=t,p=30)&       0.67 &      0.659 &       0.63 &      0.641 &      0.676 &     0.6552 & 0.01938298 \\
\hline
\end{tabular}  

\end{table}

\FloatBarrier

\textbf{Dataset: webpage}

\begin{table}[h!]
    \centering
    \caption{F1-Scores for the lr model for 5 runs of 10-fold cross validation for webpage dataset}
    \begin{tabular}{p{30.5mm}|c|c|c|c|c|c|c}
\hline
Oversampling models &        lr1 &        lr2 &        lr3 &        lr4 &        lr5 &       mean &         sd \\
\hline
      Base &      0.747 &      0.744 &      0.757 &      0.751 &       0.74 &     0.7478 & 0.00653452 \\
\hline
     SMOTE &      0.913 &      0.093 &      0.093 &      0.093 &      0.092 &     0.2568 & 0.36682721 \\
\hline
       bl1 &      0.112 &       0.11 &       0.11 &      0.107 &       0.11 &     0.1098 & 0.00178885 \\
\hline
       bl2 &      0.079 &      0.081 &      0.082 &      0.082 &      0.079 &     0.0806 & 0.00151658 \\
\hline
       SVM &      0.118 &      0.116 &      0.118 &      0.117 &      0.117 &     0.1172 & 0.00083666 \\
\hline
       ADA &      0.093 &      0.096 &      0.095 &      0.094 &      0.093 &     0.0942 & 0.00130384 \\
\hline
     LoRAS (Em=r,p=NA)&      0.098 &      0.101 &      0.098 &      0.102 &      0.095 &     0.0988 & 0.00277489 \\
\hline
F1-Score average &            &            &            &            &            & 0.21502857 &            \\
\hline
\end{tabular}  

\end{table}
\vspace{-4mm}
\begin{table}[h!]
    \centering
    \caption{Balanced accuracies for the lr model for 5 runs of 10-fold cross validation for webpage dataset}
    \begin{tabular}{p{36.5mm}|c|c|c|c|c|c|c}
\hline
Oversampling models &        lr1 &        lr2 &        lr3 &        lr4 &        lr5 &       mean &         sd \\
\hline
      Base &       0.83 &      0.832 &      0.836 &      0.839 &      0.828 &      0.833 & 0.00447214 \\
\hline
     SMOTE &      0.709 &      0.714 &      0.715 &      0.717 &      0.713 &     0.7136 & 0.00296648 \\
\hline
       bl1 &      0.768 &      0.763 &      0.763 &      0.756 &      0.762 &     0.7624 & 0.00427785 \\
\hline
       bl2 &      0.663 &      0.672 &      0.675 &      0.676 &       0.66 &     0.6692 & 0.00725948 \\
\hline
       SVM &       0.78 &      0.776 &      0.781 &      0.778 &      0.778 &     0.7786 & 0.00194936 \\
\hline
       ADA &      0.717 &      0.724 &      0.721 &      0.718 &      0.716 &     0.7192 & 0.00327109 \\
\hline
     LoRAS (Em=r,p=NA)&      0.729 &      0.738 &      0.729 &      0.738 &      0.719 &     0.7306 & 0.00789303 \\
\hline
\end{tabular}  

\end{table}
\vspace{-4mm}

\begin{table}[h!]
    \centering
    \caption{F1-Scores for the svm model for 5 runs of 10-fold cross validation for webpage dataset}
    \begin{tabular}{p{30.5mm}|c|c|c|c|c|c|c}
\hline
Oversampling models &       svm1 &       svm2 &       svm3 &       svm4 &       svm5 &       mean &         sd \\
\hline
      Base &      0.729 &      0.732 &      0.749 &      0.747 &      0.737 &     0.7388 & 0.00889944 \\
\hline
     SMOTE &      0.087 &      0.088 &      0.088 &      0.088 &      0.089 &      0.088 & 0.00070711 \\
\hline
       bl1 &      0.106 &      0.107 &      0.106 &      0.106 &      0.107 &     0.1064 & 0.00054772 \\
\hline
       bl2 &      0.106 &      0.107 &      0.106 &      0.106 &      0.107 &     0.1064 & 0.00054772 \\
\hline
       SVM &      0.118 &      0.118 &      0.119 &      0.117 &      0.118 &      0.118 & 0.00070711 \\
\hline
       ADA &      0.091 &      0.083 &      0.092 &      0.091 &      0.091 &     0.0896 & 0.00371484 \\
\hline
     LoRAS (Em=r,p=NA)&       0.09 &      0.091 &      0.093 &      0.095 &       0.09 &     0.0918 & 0.00216795 \\
\hline
F1-Score average &            &            &            &            &            & 0.19128571 &            \\
\hline
\end{tabular}  

\end{table}

\vspace{-4mm}

\begin{table}[h!]
    \centering
    \caption{Balanced accuracies for the svm model for 5 runs of 10-fold cross validation for webpage dataset}
    \begin{tabular}{p{36.5mm}|c|c|c|c|c|c|c}
\hline
Oversampling models &       svm1 &       svm2 &       svm3 &       svm4 &       svm5 &       mean &         sd \\
\hline
      Base &      0.817 &      0.814 &      0.832 &      0.823 &      0.825 &     0.8222 & 0.00704982 \\
\hline
     SMOTE &      0.693 &      0.699 &      0.697 &      0.697 &      0.702 &     0.6976 & 0.00328634 \\
\hline
       bl1 &      0.752 &      0.754 &      0.752 &      0.752 &      0.757 &     0.7534 & 0.00219089 \\
\hline
       bl2 &      0.752 &      0.754 &      0.752 &      0.752 &      0.757 &     0.7534 & 0.00219089 \\
\hline
       SVM &      0.779 &       0.78 &      0.782 &      0.778 &       0.78 &     0.7798 & 0.00148324 \\
\hline
       ADA &      0.708 &      0.715 &      0.714 &      0.708 &       0.71 &      0.711 & 0.00331662 \\
\hline
     LoRAS (Em=r,p=NA)&        0.7 &      0.708 &      0.712 &       0.85 &      0.702 &     0.7344 & 0.06479815 \\
\hline
\end{tabular}  

\end{table}

\vspace{-4mm}
\begin{table}[h!]
    \centering
    \caption{F1-Scores for the knn model for 5 runs of 10-fold cross validation for webpage dataset}
   \begin{tabular}{p{30.5mm}|c|c|c|c|c|c|c}
\hline
Oversampling models &       knn1 &       knn2 &       knn3 &       knn4 &       knn5 &       mean &         sd \\
\hline
      Base &       0.71 &      0.709 &      0.715 &      0.707 &      0.714 &      0.711 & 0.00339116 \\
\hline
     SMOTE &      0.268 &       0.27 &      0.264 &      0.268 &      0.269 &     0.2678 & 0.00228035 \\
\hline
       bl1 &      0.274 &      0.275 &      0.272 &      0.278 &      0.274 &     0.2746 & 0.00219089 \\
\hline
       bl2 &      0.291 &      0.287 &      0.285 &      0.287 &      0.287 &     0.2874 & 0.00219089 \\
\hline
       SVM &      0.269 &      0.268 &      0.266 &      0.267 &      0.268 &     0.2676 & 0.00114018 \\
\hline
       ADA &      0.266 &      0.267 &      0.261 &      0.265 &      0.265 &     0.2648 & 0.00228035 \\
\hline
     LoRAS (Em=r,p=NA)&       0.62 &      0.614 &      0.609 &       0.61 &      0.616 &     0.6138 & 0.00449444 \\
\hline
F1-Score average &            &            &            &            &            & \textbf{0.38385714} &            \\
\hline
\end{tabular}  
\end{table}
\vspace{-4mm}
\begin{table}[h!]
    \centering
    \caption{Balanced accuracies for the knn model for 5 runs of 10-fold cross validation for webpage dataset}
    \begin{tabular}{p{36.5mm}|c|c|c|c|c|c|c}
\hline
Oversampling models &       knn1 &       knn2 &       knn3 &       knn4 &       knn5 &       mean &         sd \\
\hline
      Base &      0.804 &      0.804 &      0.808 &      0.806 &      0.806 &     0.8056 & 0.00167332 \\
\hline
     SMOTE &      0.906 &      0.908 &      0.904 &      0.907 &      0.907 &     0.9064 & 0.00151658 \\
\hline
       bl1 &        0.9 &      0.901 &        0.9 &      0.903 &      0.903 &     0.9014 & 0.00151658 \\
\hline
       bl2 &      0.905 &      0.903 &      0.902 &      0.901 &      0.905 &     0.9032 & 0.00178885 \\
\hline
       SVM &      0.904 &      0.905 &      0.904 &      0.905 &      0.906 &     0.9048 & 0.00083666 \\
\hline
       ADA &      0.903 &      0.905 &      0.901 &      0.904 &      0.906 &     0.9038 & 0.00192354 \\
\hline
     LoRAS (Em=r,p=NA)&      0.924 &      0.919 &      0.921 &      0.924 &      0.928 &     0.9232 & 0.00342053 \\
\hline
\end{tabular}  
\end{table}
\vspace{-4mm}
\FloatBarrier
\textbf{Dataset: wine\_quiality}

\begin{table}[h!]
    \centering
    \caption{F1-Scores for the lr model for 5 runs of 10-fold cross validation for wine\_quiality dataset}
    %

\end{table}
\vspace{-4mm}
\begin{table}[h!]
    \centering
    \caption{Balanced accuracies for the lr model for 5 runs of 10-fold cross validation for wine\_quiality dataset}
    %
  

\end{table}

\vspace{-4mm}

\begin{table}[h!]
    \centering
    
    \caption{F1-Scores for the svm model for 5 runs of 10-fold cross validation for wine\_quiality dataset}
    %
  

\end{table}

\vspace{-4mm}

\begin{table}[h!]
    \centering
    \caption{Balanced accuracies for the svm model for 5 runs of 10-fold cross validation for wine\_quiality dataset}
    %
  

\end{table}

\vspace{-4mm}

\begin{table}[h!]
    \centering
    \caption{F1-Scores for the knn model for 5 runs of 10-fold cross validation for wine\_quiality dataset}
    %

\end{table}
\vspace{-4mm}

\begin{table}[h!]
    \centering
    \caption{Balanced accuracies for the knn model for 5 runs of 10-fold cross validation for wine\_quiality dataset}
   %
  

\end{table}

\FloatBarrier
\textbf{Dataset: yeast\_ml8}

\begin{table}[h!]
    \centering
    \caption{F1-Scores for the lr model for 5 runs of 10-fold cross validation for yeast\_ml8 dataset}
   %
  

\end{table}

\vspace{-4mm}

\begin{table}[h!]
    \centering
    \caption{Balanced accuracies for the lr model for 5 runs of 10-fold cross validation for yeast\_ml8 dataset}
    %
  

\end{table}

\vspace{-4mm}
\begin{table}[h!]
    \centering
    \caption{F1-Scores for the svm model for 5 runs of 10-fold cross validation for yeast\_ml8 dataset}
    %

\end{table}

\vspace{-4mm}

\begin{table}[h!]
    \centering
    \caption{Balanced accuracies for the svm model for 5 runs of 10-fold cross validation for yeast\_ml8 dataset}
    %
  

\end{table}

\vspace{-4mm}

\begin{table}[h!]
    \centering
    \caption{F1-Scores for the knn model for 5 runs of 10-fold cross validation for yeast\_ml8 dataset}
    %
  

\end{table}

\vspace{-4mm}
\begin{table}[h!]
    \centering
    \caption{Balanced accuracies for the knn model for 5 runs of 10-fold cross validation for yeast\_ml8 dataset}
    %
  

\end{table}

\FloatBarrier

\textbf{Dataset: yeast\_me2}

\begin{table}[h!]
    \centering
    \caption{F1-Scores for the lr model for 5 runs of 10-fold cross validation for yeast\_me2 dataset}
    %
  
\end{table}

\vspace{-4mm}

\begin{table}[h!]
    \centering
    \caption{Balanced accuracies for the lr model for 5 runs of 10-fold cross validation for yeast\_me2 dataset}
    %
  

\end{table}

\vspace{-4mm}
\begin{table}[h!]
    \caption{F1-Scores for the svm model for 5 runs of 10-fold cross validation for yeast\_me2 dataset}
    %
  
\end{table}

\vspace{-4mm}
\begin{table}[h!]
    \centering
    \caption{Balanced accuracies for the svm model for 5 runs of 10-fold cross validation for yeast\_me2 dataset}
    %
  

\end{table}

\vspace{-4mm}
\begin{table}[h!]
    \centering
    \caption{F1-Scores for the knn model for 5 runs of 10-fold cross validation for yeast\_me2 dataset}
    %
  

\end{table}

\vspace{-4mm}

\begin{table}[h!]
    \centering
    \caption{Balanced accuracies for the knn model for 5 runs of 10-fold cross validation for yeast\_me2 dataset}
    %
  

\end{table}

\FloatBarrier
\textbf{Dataset: credit fraud}
\begin{table}[h!]
    \centering

    \caption{F1-Scores for the lr model for 5 runs of 10-fold cross validation for credit fraud dataset}
    %
  

\end{table}

\vspace{-4mm}

\begin{table}[h!]
    \centering
    \caption{Balanced accuracies for the lr model for 5 runs of 10-fold cross validation for credit fraud dataset}
    %
  
\end{table}

\vspace{-4mm}

\begin{table}[h!]
    \centering
    \caption{F1-Scores for the rf model for 5 runs of 10-fold cross validation for credit fraud dataset}
%
  
\end{table}

\vspace{-4mm}

\begin{table}[h!]
    \centering
    \caption{Balanced accuracies for the rf model for 5 runs of 10-fold cross validation for credit fraud dataset}
%
  

\end{table}

\end{document}